\documentclass{article}

\usepackage[nonatbib,final]{neurips_2021}

\usepackage[american]{babel}
\usepackage[utf8]{inputenc} %
\usepackage[T1]{fontenc}    %
\usepackage{csquotes}
\usepackage{hyperref}       %
\usepackage{url}            %
\usepackage{booktabs}       %
\usepackage{amsfonts}       %
\usepackage{nicefrac}       %
\usepackage{microtype}      %
\usepackage{xcolor}         %
\usepackage{array}          %
\usepackage{tikz}
\usetikzlibrary{calc,positioning,decorations.pathreplacing,shapes}

\usepackage{amsmath}
\usepackage{amsthm}
\usepackage{amssymb}
\PassOptionsToPackage{normalem}{ulem}
\usepackage{ulem}

\usepackage[hide=true,setmargin=true,marginparwidth=1.2in]{marginalia}

\usepackage{enumitem}
\usepackage{xspace}
\usepackage[noabbrev, capitalise]{cleveref}
\usepackage{graphicx}
\usepackage{subcaption}

\usepackage[%
		minnames=1,maxnames=99,maxcitenames=2,
		style=numeric-comp,
		sorting=none,
		sortcites=false, %
		doi=true,
		url=true,
		uniquename=init,
		giveninits=true,
		backend=biber,
		hyperref=auto,
		natbib=true]{biblatex}

\DeclareFieldFormat{bibentrysetcount}{\mkbibparens{\mknumalph{#1}}}
\DeclareFieldFormat{labelnumberwidth}{\mkbibbrackets{#1}}

\renewbibmacro{in:}{%
	\ifentrytype{article}{}{\printtext{\bibstring{in}\intitlepunct}}}
\renewbibmacro{in:}{%
	\ifentrytype{inproceedings}{}{\printtext{\bibstring{in}\intitlepunct}}}

\DeclareNameAlias{sortname}{family-given}

\theoremstyle{plain}
\newtheorem{thm}{\protect\theoremname}
  \theoremstyle{plain}
  \newtheorem{prop}[thm]{\protect\propositionname}
  \theoremstyle{plain}
  \newtheorem{cor}[thm]{\protect\corollaryname}
  \theoremstyle{remark}
  \newtheorem{rem}[thm]{\protect\remarkname}
  \theoremstyle{plain}
  \newtheorem{conjecture}[thm]{\protect\conjecturename}
  \theoremstyle{plain}
  \newtheorem{lem}[thm]{\protect\lemmaname}
  \theoremstyle{plain}

\makeatother

\usepackage{mathtools}

\usepackage{contour}
\usepackage[normalem]{ulem}

\contourlength{0.8pt}

\renewcommand{\underline}[1]{%
  \uline{\phantom{#1}}%
  \llap{\contour{white}{#1}}%
}

\definecolor{mylinenocolor}{rgb}{0.65,0.65,0.65}

\definecolor{OliveGreen}{rgb}{0,0.6,0}
\newcommand{\mufan}[1]{#1}

\usepackage{babel}
  \providecommand{\conjecturename}{Conjecture}
  \providecommand{\corollaryname}{Corollary}
  \providecommand{\lemmaname}{Lemma}
  \providecommand{\propositionname}{Proposition}
  \providecommand{\remarkname}{Remark}
\providecommand{\theoremname}{Theorem}

\global\long\def\p{\mathbf{P}}

\global\long\def\cov{\mathbf{Cov}}
\global\long\def\var{\mathbf{Var}}
\global\long\def\corr{\mathbf{Corr}}
\global\long\def\e{\mathbf{E}}
\global\long\def\one{\mathtt{1}}

\global\long\def\norm#1{\left\Vert #1\right\Vert }
\global\long\def\abs#1{\left|#1\right|}
\global\long\def\given#1{\left|#1\right.}

\global\long\def\bN{\mathbb{N}}

\global\long\def\bR{\mathbb{R}}
\global\long\def\bS{\mathbb{S}}

\global\long\def\cA{\mathcal{A}}

\global\long\def\cN{\mathcal{N}}

\global\long\def\diag{\text{diag}}

\global\long\def\half{\frac{1}{2}}

\global\long\def\al{\alpha}
\global\long\def\be{\beta}

\global\long\def\ep{\epsilon}

\global\long\def\th{\theta}

\global\long\def\la{\lambda}

\global\long\def\si{\sigma}

\global\long\def\vp{\varphi}
\global\long\def\vpp{\varphi_+}
\global\long\def\ch{\chi}

\global\long\def\dequal{\stackrel{d}{=}}
\global\long\def\defequal{\coloneqq}

\global\long\def\di{\partial}

\global\long\def\nin{{n_\text{in}}}
\global\long\def\nout{{n_\text{out}}}
\global\long\def\zout{z^{\text{out}}}
\global\long\def\Wout{W^{\text{out}}}
\global\long\def\opon{\left(1+O\left(\frac{1}{n}\right)\right)}

\newcommand{\defn}[1]{\emph{#1}}
\newcommand*\samethanks[1][\value{footnote}]{\footnotemark[#1]}

\title{The Future is Log-Gaussian: ResNets
and Their Infinite-Depth-and-Width Limit at Initialization}

\author{%
  Mufan (Bill)~Li\thanks{Equal contribution authors.} \\
  University of Toronto,\\
  Vector Institute \\
  \texttt{mufan.li@mail.utoronto.ca} 
   \And
   Mihai~Nica\samethanks \\
   University of Guelph,\\
   Vector Institute \\
  \texttt{nicam@uoguelph.ca} 
   \And
   Daniel M.~Roy \\
   University of Toronto,\\
   Vector Institute \\
  \texttt{droy@utstat.toronto.edu} 
}

\begin{document}

\maketitle

\begin{abstract}
Theoretical results show that neural networks can be approximated by Gaussian processes in the infinite-width limit. However, for fully connected networks, it has been previously shown that for any fixed network width, $n$, the Gaussian approximation gets worse as the network depth, $d$, increases. Given that modern networks are deep, this raises the question of how well modern architectures, like ResNets, are captured by the infinite-width limit. To provide a better approximation, we study ReLU ResNets in the infinite-depth-and-width limit, where \emph{both} depth and width tend to infinity as their ratio, $d/n$, remains constant. In contrast to the Gaussian infinite-width limit, we show theoretically that the network exhibits log-Gaussian behaviour at initialization in the infinite-depth-and-width limit, with parameters depending on the ratio $d/n$. Using Monte Carlo simulations, we demonstrate that even basic properties of standard ResNet architectures are poorly captured by the Gaussian limit, but remarkably well captured by our log-Gaussian limit. Moreover, our analysis reveals that ReLU ResNets at initialization are hypoactivated: fewer than half of the ReLUs are activated. Additionally, we calculate the interlayer correlations, which have the effect of exponentially increasing the variance of the network output. Based on our analysis, we introduce \emph{Balanced ResNets}, a simple architecture modification, which eliminates hypoactivation and interlayer correlations and is more amenable to theoretical analysis.
\end{abstract}

\section{Introduction}

The characterization of infinite-width dynamics of gradient descent (GD)
in terms of the so-called Neural Tangent Kernel (NTK) 
\citep{jacot2018neural,du2019gradient,allen2019convergence,zou2020gradient,chizat2019lazy,lee2019wide,yang2019scaling,yang2020tensor,arora2019exact,chen2021how}
represented a major breakthrough in our understanding of %
deep learning in the large-width regime.
Before the identification of infinite-width limits,
the theoretical study of deep learning had long been hindered by the apparent analytical intractability of gradient descent and variants acting on the nonconvex objectives used to train neural networks.
Despite this progress, evidence suggests that deep neural networks can outperform their infinite-width limits in practice \citep{arora2019harnessing}, particularly when the depth of the network is large.
These observations motivate the study of other approximations that may close the gap.

Several alternative limits have been proposed. Around the time of the discovery of the NTK limit,
mean-field limits were also characterized \citep{RVE18,cb18,SS18, MeiE7665}, and more recently have been linked with the NTK limit \citep{mei19}.
\citet{featurelearning} describe a family of infinite-width limits indexed by the scaling limits of initial weight variance, weight rescaling, and learning rates.
This family includes both the NTK and mean field limits. One motivation for studying these alternative limits is that they yield a notion of feature learning, which provably does not occur in the NTK limit \citep{featurelearning}.

Despite the variety of these limits, one common feature is that the \emph{depth} of the network (i.e., the number of its layers) is treated as a constant as the width of the network is allowed to grow.
Indeed, for fixed width, $n$, the Gaussian approximations at initialization \citep{neal1995bayesian,yang2017mean,lee2018deep,de2018gaussian,novak2018bayesian,yang2019tensor}
worsen as the depth, $d$, increases.
While real-world networks are fairly wide, their relative depth is not trivial.
\citet{hanin2019products} were the first to compute an infinite-depth-and-width limit for fully connected networks. While the training dynamics of this limit are still not completely understood, we now know that, in
the infinite-depth-and-width limit, the neural tangent kernel is random, and the derivative of the kernel is nonzero at initialization. This means training does not correspond to that of a linear model, like it does in the NTK limit
\citep{hanin2019finite,seleznova2020analyzing}.

\begin{figure}[t]
\centering
\includegraphics[width=0.95\linewidth]{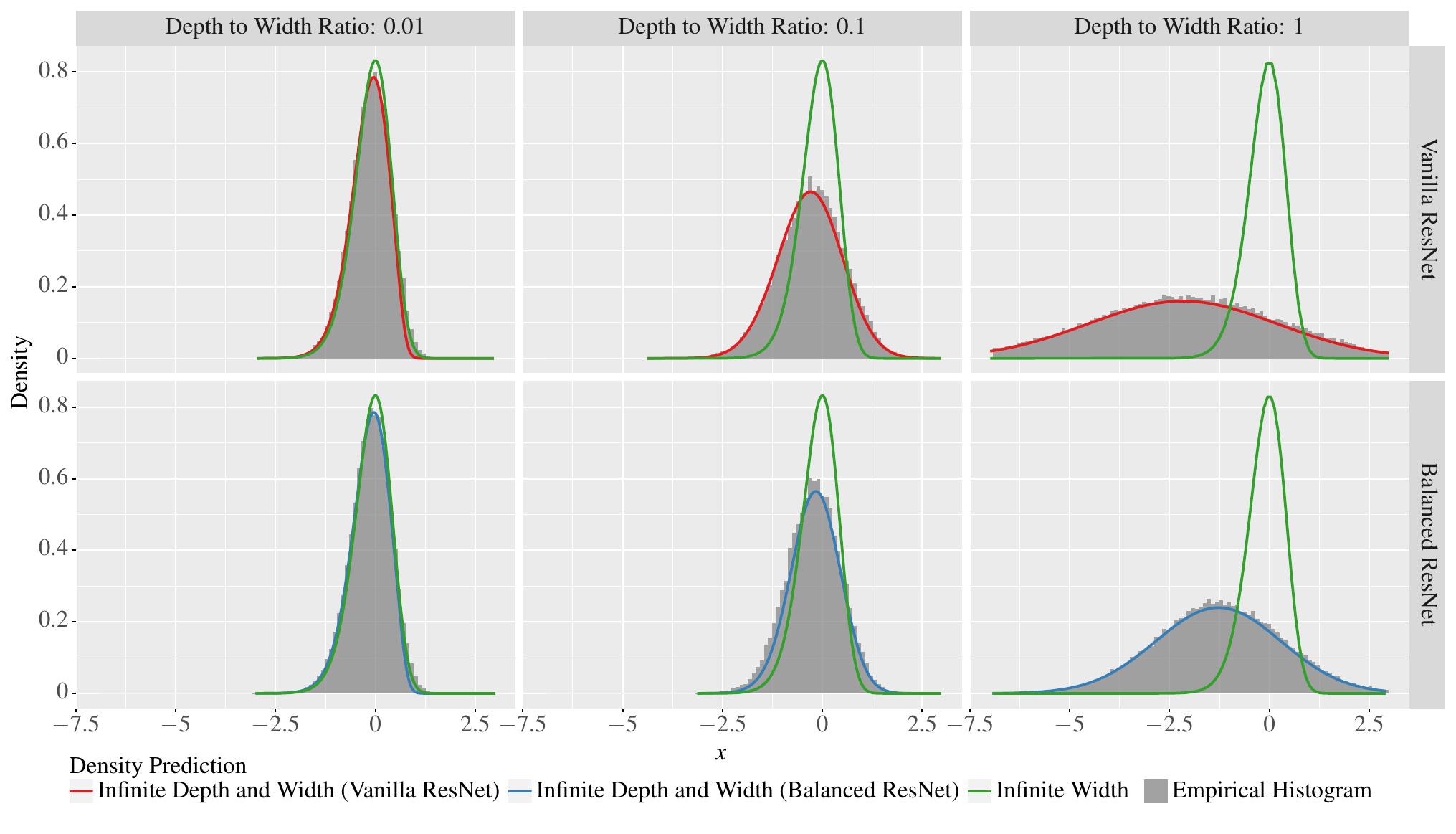}
\caption{ {\bf Probability density function of } $\ln \norm{ \zout}^2$ for six network configurations on initialization. All networks have $n=100, \nin=\nout=10, \alpha = \lambda = 1/\sqrt{2}$. Depth $d$ varies by column. Top row: Vanilla ReLU ResNets. Bottom Row: Balanced ResNets (\cref{subsec:balanced_resnets}), which randomize the non-linearities used at each neuron thereby exponentially reducing the variance. The theoretical curves shown are the result of \cref{prop:main} and \cref{prop:balanced_resnets_result}. These predictions converge to the infinite-width prediction when $d/n \to 0$.
The details of the simulations and plot here is in \cref{sec:experiment}.
}
\label{fig:density_merge}
\end{figure}

In addition to these theoretical corrections to the Gaussian limit, practitioners have begun to notice that even basic properties of standard neural networks do not line up with those predicted by the infinite-width limit. Perhaps the most basic is that the gradients of the network early in training are not Gaussian, but instead are approximately log-Gaussian \citep{lognormalgrad}. 
In fact, one should see log-Gaussian behaviour agrees with the theoretical predictions of \citet{hanin2019products}, although there has yet to be a careful empirical comparison made between the precise predictions coming from infinite-depth-and-width models and real world networks.

In practice, however, fully connected networks are not often used without architectural modifications. In particular, residual connections, ushered into widespread use after the description of the ResNet architecture \citep{he2016deep}, produced very deep architectures that were practically useful with optimization techniques available at the time. 
\mufan{Initialization schemes for ResNets have been studied in the infinite-width limit \citep{yang2017mean,hayou2021stable} or with modifications \citep{zhang2018fixup,blumenfeld2020beyond,arpit2019initialize}.}

\mufan{In this work, we consider the infinite-depth-and-width limit of fully connected architectures with residual connections (``Vanilla ResNets'') and standard initializations. Here, the analysis is complicated by the effect of skip connections,} 
which introduce interlayer correlations that have a non-negligible effect in this limit. 
\mufan{Surprisingly, we observe a counter-intuitive but fundamental phenomenon, whereby these skip connection cause the network to be \textbf{hypoactivated}, meaning that less than half of the neurons are activated on initialization. This fact undermines key assumptions that underlie other infinite-depth-and-width limit studies of architectures without residual connections or with nonstandard modifications, such as post-activation residual connections.}

\mufan{Hypoactivation is a roadblock that all theoretical research into standard ResNet architectures must contend with---it is an unavoidable property of the standard architecture and the root of many technical difficulties.} 
\mufan{In order to sidestep this roadblock, we introduce a conjecture that bounds the size of hypoactivation and effect on interlayer correlation. The conjecture is inspired by empirical evidence from Monte Carlo simulations, as well as several simplified analyses that ignore certain technical difficulties. The conjecture introduces what we believe to be the minimal assumption necessary to allow rigorous theoretical work to proceed. It also defines an important open problem in the study of limits of residual architectures.}

\mufan{To demonstrate the utility of the conjecture, we 
show that it leads to precise predictions.
In particular, we prove a limit theorem characterizing the exact marginal distribution of the output at initialization in the infinite-depth-and-width limit, up to order $O(dn^{-2})$.} 
Our limit result shows that ResNets have
\defn{log-Gaussian behaviour}
on initialization, and like fully connected networks \citep{hanin2019products}, the behaviour is determined by the depth-to-width aspect ratio. This corroborates recent empirical observations about deep ResNets \citep{lognormalgrad}.
Since real world networks are finite, the question of how well this approximates finite behaviour is of paramount importance. Based on Monte Carlo simulations, we find excellent agreement between our predictions and finite networks (see \cref{fig:density_merge}). 
Moreover, for very deep networks (e.g. $d/n=1$) the infinite-depth-and-width prediction is extremely different than the infinite-width prediction. More surprisingly, however, is that even at comparatively small depth-to-width ratio (e.g. $d/n = 0.1$) the two limits are already significantly different. 
Furthermore, we also observe that the effects due to hypoactivation and interlayer correlation are non-negligible; these effects are precisely the difference between Vanilla ResNets and so-called ``Balanced Resnets'' in \cref{fig:density_merge}.
Perhaps most importantly, we observe that real network outputs exhibit exponentially larger variance than predicted by infinite-width limits. This type of variance at initialization is known to cause exploding and vanishing gradients and other types of training failures \cite{hanin2018start,hanin2018neural}.
Our result also implies that the output neurons of the network are not independent as predicted by the Gaussian infinite-width limit. See \cref{fig:out_corr}.

In order to maintain the same skip connections from layer to layer, but render the activation patterns completely independent from layer to layer on initialization,
we introduce the Balanced ResNet architecture,
where the sign of each neuron's activation function is randomized. We demonstrate that this exponentially decreases the variance of the network output on initialization. Moreover, this independence between neurons also makes the model more amenable to  theoretical analysis and opens the door to future understanding of network behaviour. Although it is beyond the scope of this paper, a small preliminary empirical study (\cref{sec:experiment}) suggests that standard training regimes are not negatively affected by replacing standard ResNet architectures with Balanced ones.

We summarize our main contributions as follows: 
\begin{itemize}[leftmargin=1em,itemsep=2pt,parsep=2pt]    %
    \item \mufan{We identify and characterize a fundamental property of ResNets that we call \textbf{hypoactivation}: less than half of the ReLU neurons are activated. Based on empirical evidence, we formulate a precise minimal conjecture bounding the effect of hypoactivation that 
    permits us to make precise, rigorous estimates for other properties of  ResNets.} 
    \item We prove a limit theorem which shows that the output of ResNets on initialization exhibits {\bf log-Gaussian} behaviour with parameterization depending on the depth-to-width ratio $d/n$. %
    \item We provide {\bf empirical evidence} from Monte Carlo simulations showing our theory provides more accurate predictions for simple properties of finite networks
    compared to the predictions made by the Gaussian infinite-width limit.
     \item We introduce the \textbf{\defn{Balanced ResNet} architecture},
    which corrects the hypoactivation and variance due to layerwise correlation from Vanilla ResNets. We also prove that the output for this architecture is log-Gaussian on initialization with exponentially lower variance. This simple modification can be applied to \emph{any} neural network that uses ReLU activations.
\end{itemize}
\nopagebreak

\nopagebreak
\section{Main Results}
\label{sec:main_results}

In terms of the notation in Table~1, a {\bf Vanilla ResNet} with fully connected first/last layers and $d$ hidden layers of width $n$ is defined by
\begin{equation}\label{eq:defn_resnet}
z^{0} \defequal \frac{1}{\sqrt{\nin}}W^0 x, \hspace{0.8em} z^{\ell} \defequal\al z^{\ell-1}+\la \sqrt{\frac{2}{n}} W^\ell {\vpp\left(z^{\ell-1}\right)}\text{ for }1\leq \ell\leq d, \hspace{0.8em}
\zout \defequal \frac{1}{\sqrt{n}}\Wout z^d \,.
\end{equation}
Note that factors of $\sqrt{2n^{-1}}$ in the hidden layer are equivalent to intializing according to the so-called He initialization \citep{he2015delving}. Other intializations correspond to changing the coefficient $\la$. This setup is similar to that of ``Stable ResNets'' \citep{hayou2021stable}, where the infinite-width limit is studied.

\begin{table}[t]\label{tab:notations}
\begin{tikzpicture}%
\node[draw,anchor=north east,inner sep=0,outer sep=0,inner sep=3pt] (tab) at (0,0) {
\small\centering
\begin{tabular}{>{\raggedright}p{0.11\textwidth}>{\raggedright}p{0.245\textwidth}>{\raggedright}p{0.145\textwidth}>{\raggedright}p{0.355\textwidth}}
\textbf{Notation} & \textbf{Description} & \textbf{Notation} & \textbf{Description}\tabularnewline[2pt]
$\nin \in \bN$ & Input dimension & $\nout \in \bN$ & Output dimension \tabularnewline
$n \in \bN $ & Hidden layer width & $d \in \bN$ & Number of hidden layers (depth) \tabularnewline
$\vpp(\cdot)$ & ReLU function\\ \ \ $\vpp(x) = \max\{x,0\}$ &
$\vp_{-}(\cdot)$ & ``Domain Flipped'' ReLU\\ \ \ $\vp_-(x) = \max\{-x,0\}$\tabularnewline
$\alpha \in \bR$ & Skip connection coefficient  & $\lambda \in \bR^+$ & Feed-forward coefficient \tabularnewline
$x \in \bR^\nin$ & Input  & $W^0 \in \bR^{\nin \times n}$ & Weight matrix at layer 0
\tabularnewline
$\zout \in \bR^\nout$ & Network output & \rlap{\smash{$\Wout \in \bR^{n \times \nout}$}} & Weight matrix at final layer.
\tabularnewline[-1.2em]
$z^\ell \in \bR^n$ & Neurons (pre-activation) \\ \ \ for layer $1\leq \ell \leq d$
  &  $W^\ell \in \bR^{n \times n}$ & Weight matrix at layer $1 \leq \ell \leq d$\\ \ \ all weights initialized i.i.d.
$\sim \cN(0,1)$
\end{tabular}
};
\node[draw,align=right,anchor=north east,fill=white,outer sep=0] (tabcap) at (0,0) {\scalebox{0.85}{Table 1: Notation}};
\end{tikzpicture}
\end{table}
In the infinite-depth-and-width limit,
the intuition that half of the ReLU units are active (i.e. nonzero) because of symmetry is surprisingly not correct.
We find that the following quantities play an important role. We define the {\bf average hypoactivation (of layer $\ell$)}
and the {\bf total hypoactivation of the network} by %
\begin{equation} \label{eq:hypoactivation}
    h_\ell \defequal \e\left[\norm{\vpp\left(\hat{z}^\ell\right)}^2\right] - \half \,,
    \qquad
    \textstyle
    h_\text{total} \defequal \sum_{\ell=1}^d h_\ell \,,
\end{equation}
where $\hat{z}^\ell \defequal z^\ell/\norm{z^\ell}$ and $\e$ means expectation over the choice of random network weights on initialization.
The average hypoactivation is a measure of how many ReLU neurons are activated in layer $\ell$; $h_\ell = 0$ indicates roughly half of the neurons are active.
Counter-intuitively, we observe that in a vanilla ResNet, $h_\ell$ is \emph{negative} and $|h_\ell| = O(1/n)$,
indicating slightly less than half the neurons are active.\footnote{
For a quantity $f=f(n,d)$ whose dependence on width and depth may be implicit,
we use the notation $f=O(d^a/n^b)$ to mean that, for all choice of constants $\al, \la, r_{-},r_{+} > 0$, there exists a constant
$C >0$
such that $\abs{f(n,d)}\leq C {d^a}/{n^b}$ for all $d,n$ where $r_{-} < {d}/{n} < r_{+}$.
This notation will allow us to state precise limit theorems when $d,n\to\infty$ with the ratio ${d}/{n}$ converging to a constant.
}
After compounding over $d$ layers, the total hypoactivation is of order $h_{\text{total}} = O(d/n) = O(1)$ in the infinite-depth-and-width limit. As we will see, this effect has a non-trivial contribution.

At the same time, we also find the covariance between the activations of various layers does not vanish in the infinite-depth-and-width limit.
This motivates the definition of the {\bf total interlayer covariance correction}
\begin{equation}
    I_\text{total}
    \defequal \textstyle{\sum_{1 \le \ell \neq \ell^{\prime} \le d}}\,
    \cov\left( 2\norm{\vpp\big(\hat{z}^\ell\big)}^2, 2\|\vpp\big(\hat{z}^{\ell^\prime}\big)\|^2 \right).
\end{equation}
As with the hypoactivation, skip connections cause this term to be non-trivial.
We formulate Conjecture \ref{conj:hypoactivation},
which contains a precise encapsulation of the behaviour of $\hat{z}^\ell$ we observe.

{\bf Conjecture \ref{conj:hypoactivation} (Informal).}
In expectation, the layers $\hat{z}^\ell$ can be approximated by uniform random variables from the sphere up to a relative error of $O(1/n)$.

The conjecture is well supported by Monte-Carlo simulations (see \cref{fig:hypo_conj}).
We provide a more detailed discussion and a precise statement of \cref{conj:hypoactivation} in \cref{sec:hypo_and_corr}.
Assuming the conjecture holds, we prove a limit theorem about the distribution of $\zout$. Informally, this says that $\zout$ is {\bf approximately a log-Gaussian scalar} times an independent Gaussian vector
\begin{equation}
    \zout \approx \frac{\norm{x}}{\sqrt{\nin}} \left(\al^2 + \la^2\right) ^{\frac{d}{2}}  \exp\left( \half \cN\left(-\frac{\beta}{2} + 2c h_\text{total},\beta + c^2I_\text{total} \right)\right)  \vec{Z},
\end{equation}
where $\vec{Z}$ has iid $\cN(0,1)$ entries,
and $\beta$ and $c$ are defined by
\begin{equation} \label{eq:beta_c}
\beta \defequal \frac{2}{n} + \frac{d}{n}\cdot \frac{5\la^4 + 4\al^2 \la^2}{(\al^2+\la^2)^2}, \hspace{1em} c \defequal \frac{\la^2}{\al^2 +\la^2}.
\end{equation}
The precise statement, including asymptotic error bounds, is as follows.

\begin{thm} \label{prop:main}
For any choice of hyperparameters $\nin,\nout,n,d,\alpha,\lambda$, and every input $x$, the output $\zout$ at initialization has a marginal distribution which can be written in the form
\begin{equation}\label{eq:main}
    \zout \dequal \frac{\norm{x}}{\sqrt{\nin}} \left(\al^2 + \la^2\right) ^{\frac{d}{2}} \exp\left(\half G \right) \vec{Z},
\end{equation}
where $\vec{Z}\in \bR^{\nout}$ is a Gaussian random vector with iid $\cN(0,1)$ entries, $G=G(n,d,\al,\la)$ is a random variable which is independent of $\vec{Z}$ and whose distribution does not depend on $\nin,\nout$ or $x$.

Consider the limit where \emph{both} the network depth $d \to \infty$ and hidden layer width $n \to \infty$ in such a way that the ratio $d/n$ converges to a non-zero constant. In this limit, assuming that Conjecture \ref{conj:hypoactivation} holds, then the random variable $G \in \bR$ has the following asymptotic behaviour:
\begin{align}
\label{eq:E}
    \e\left[G\right]
    =
        -\frac{\beta}{2} + 2c h_\text{total} + O\left( \frac{d}{n^{2}}\right) , \hspace{1em}
    \var\left[G\right]
    =
        \beta + c^2 I_{\text{total}}  + O\left( \frac{d}{n^{2}}\right),
\end{align}
where $\beta$ and $c$ are as in \eqref{eq:beta_c}, and moreover $G$ converges in distribution to a Gaussian random variable with mean and variance given by \eqref{eq:E}  in this limit.

\end{thm}

The main ideas of the proof of \cref{prop:main} is given in \cref{sec:proof_idea} and the detailed proof is given in \cref{sec:supp_proofs}. We also provide more explicit formulas for $h_\text{total}, I_\text{total}$ below.

\begin{prop} \label{prop:hypo_and_cov}
Assume \cref{conj:hypoactivation} is true.
Then in the same infinite-depth-and-width limit
as \cref{prop:main},
the total hypoactivation $h_\text{total}$ and total interlayer covariance $I_\text{total}$ obey
\begin{equation}\label{eq:h_I_formulas}
    h_\text{total} = C_{\alpha,\lambda}\frac{d}{n} + O\left(\frac{d}{n^2}\right), \hspace{1em}
    I_\text{total} = \sum_{1\leq \ell\neq \ell^{\prime}\leq d} \frac{\bar{J}_{2}(\th_{|\ell^\prime - \ell|})-\bar{J}_{2}(\pi-\th_{|\ell^\prime - \ell|})}{n}+ O\left( \frac{d}{n^{2}}\right).
\end{equation}
Here $C_{\alpha,\lambda}$ is a constant depending on $\alpha,\lambda$ and
\begin{equation*}
\bar{J}_{2}(\th)
\defequal J_{2}\left(\th\right) / {\pi}
= 3\sin(\th)\cos(\th)/\pi+\left(1-\th/\pi\right)\left(1+2\cos^{2}\th\right),
\end{equation*}
where
 $J_2(\theta)$
first appeared in \citep{cho2009kernel}, and $\theta_k$ is such that $\cos(\th_k) = \al^k/(\al^2+\la^2)^{k/2}$ .
\end{prop}

\begin{figure}[t]
\centering
\begin{subfigure}[b]{0.495\textwidth}
\centering
\includegraphics[width=\textwidth]{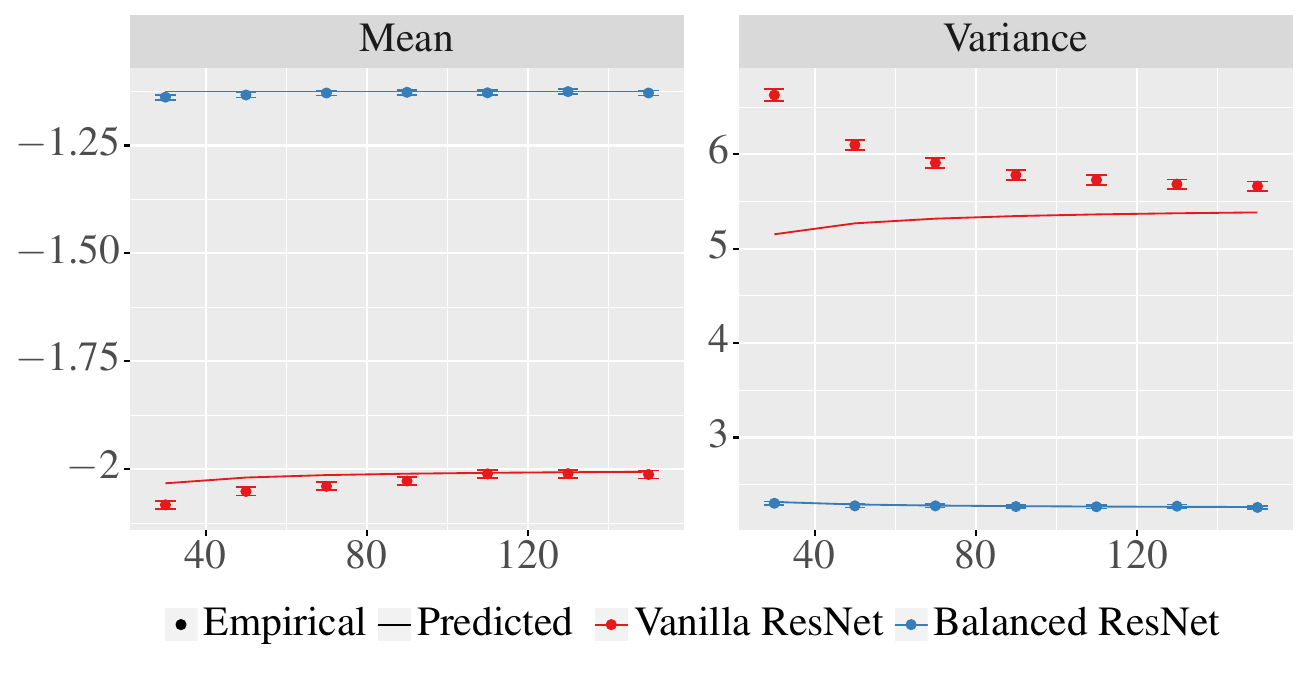}
\caption{
$\e$ and $\var$ of
$G$ from \eqref{eq:main}
and prediction.
}
\label{subfig:mean_var_fixed_ratio}
\end{subfigure}
\begin{subfigure}[b]{0.495\textwidth}
\centering
\includegraphics[width=\textwidth]{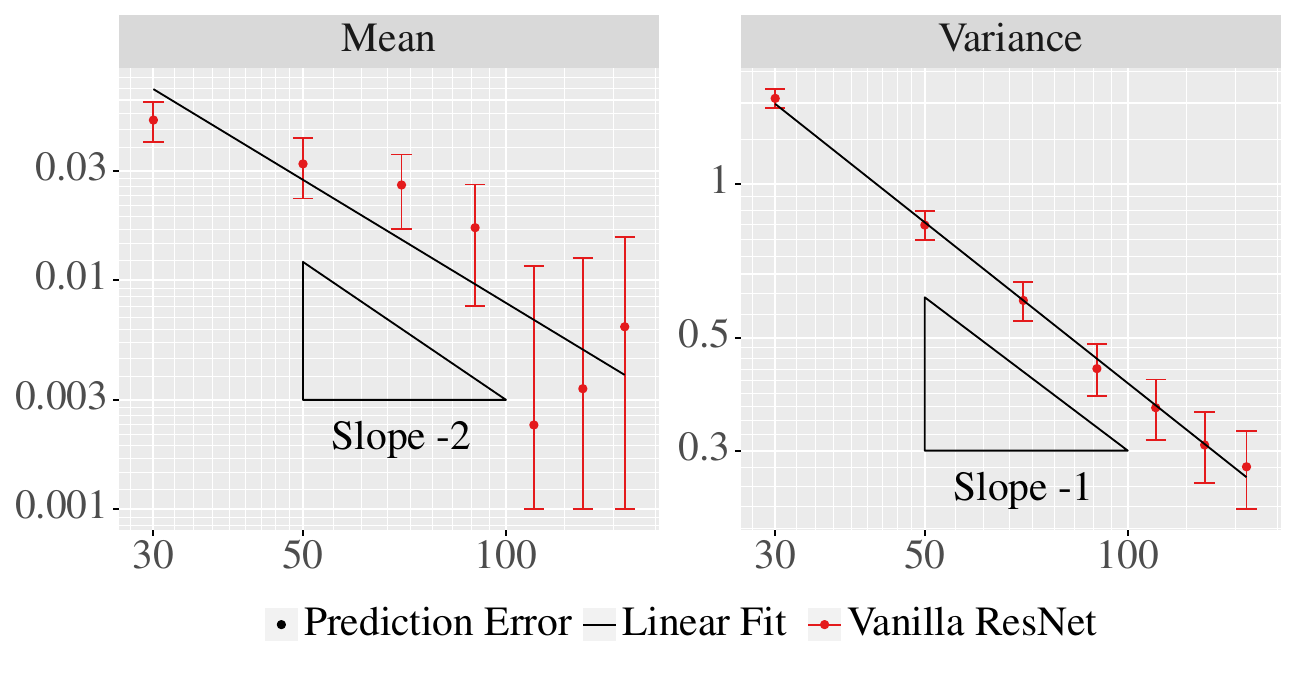}
\caption{
Prediction error for $\e$ and $\var$ of
$G$ from \eqref{eq:main}.
}
\label{subfig:mean_var_est_fixed_ratio}
\end{subfigure}
\caption{Empirical mean and variance and infinite-depth-and-width prediction for the random variable $G(n,d,\alpha,\la)$ from \eqref{eq:main} compared to the results of Theorems \ref{prop:main} \& \ref{prop:balanced_resnets_result} . For these simulations, $\alpha = \lambda = 1/\sqrt{2}, d/n=1$ are fixed and width $n$ is varied on the $x$-axis. The error bars indicate 95\% confidence interval (CI) of the Monte-Carlo simulation, with truncation at $1\mathrm{e}{-3}$ for plotting on log-scale. The balanced ResNet predictions fall within the Monte-Carlo CI and so are not plotted on log scale. Since $d/n$ is fixed, Figure \ref{subfig:mean_var_est_fixed_ratio} indicates the error in the asymptotic prediction of variance \eqref{eq:E} is $O(n^{-1})=O(d n^{-2})$ as claimed. On the other hand, the error in mean \eqref{eq:E} is $O(n^{-2})=O(d n^{-3})$; this is one order smaller than the statement proven in \cref{prop:main}. A possible explanation is that the sub-leading error term in our approximation for $G$ is mean zero.
}
\label{fig:est_fixed_ratio}
\end{figure}

\begin{rem}
The Gaussian infinite-width limit predicts that the marginals of $\zout$ should have the form of \eqref{eq:main} with $G$ being identically zero.
As $\zout$ depends exponentially on $G$, the infinite-depth-and-width limit predicts the variance is exponentially larger than the infinite-width limit. See \cref{sec:consequences} for a detailed discussion and \cref{fig:out_corr} for verification against finite networks.
\end{rem}

Using Monte Carlo simulations, we estimate the constant $C_{\alpha, \lambda}$. The result of \cref{prop:main} is then compared against finite networks in \cref{fig:est_fixed_ratio}. Full proof can be found \cref{sec:supp_proofs}.

The idea of using a scaling parameter $\alpha = 1/\sqrt{2}$ in the skip connections has been noted in empirical papers \citep{dhariwal2021diffusion} and also studied under simplified assumptions on the number of activation in each layer \citep{balduzzi2017shattered}. Our result shows that by choosing $\al^2+\la^2=1$, the prefactor in \eqref{eq:main} does not grow with depth thereby enabling deeper networks to be trained.

In the case that $\al=0$, $\la=1$, the architecture reduces to a fully connected network. In this case, log-Gaussian behaviour of the network was discovered in \citet{hanin2019products}. The fully connected case is simpler because the direction vectors $\hat{z}^\ell$ are always uniformly distributed on the unit sphere and are independent from layer to layer. This means $h_\ell = 0$ and $\th_n = 0$ which greatly simplifies the result of \cref{prop:main}.

Furthermore, the proof easily extends to the case where the coefficients $\al_\ell,\la_\ell$ vary from layer to layer; see \cref{sec:layer_dependent} for the general statement.
\citet{hayou2021stable} and \citet{hanin2018start} have studied ResNets where $\al_\ell=1$ in every layer, but $\la_\ell$ is allowed to vary.
The prefactor of our result in this case becomes $\prod_{\ell=1}^d \left(\smash{1^2 + \la_\ell^2}\right) \approx \exp \left( \sum_{\ell=1}^d \la_\ell^2 \right)$, which is akin to the behaviour found in those papers. Additionally, our result precisely quantifies the log-Gaussian behaviour and its dependence on the sequence $\la_i$ through the parameters $\be$ and $c$ in the infinite-depth-and-width limit.

\subsection{Log-Gaussian Behaviour of Balanced ResNets}
\label{subsec:balanced_resnets}

Given a collection of \emph{iid uniformly random signs}, $s^\ell_{i} \in \{+,-\},  {1\leq \ell \leq d, 1\leq i\leq n}$, a {\bf balanced ResNet} is defined much like a Vanilla ResNet, except that a random sign is applied preactivation: 
\begin{equation*}
z^{\ell} \defequal \al z^{\ell-1}+\la \sqrt{\frac{2}{n}} W^\ell {\vp_{s^\ell} \left(z^{\ell-1}\right)}\text{ for }1\leq \ell\leq d,\\
\end{equation*}
where at each layer the vector function $\vp_{s^\ell}:\bR^n \to \bR^n$ applies either $\vpp$ or $\vp_-$ to the entries according to the random signs $s^\ell$. More precisely, the $i$-th component is
$$\vp_{s^\ell}(z)_i \defequal
\begin{cases}
			\vpp(z_i) = \max(z_i, 0), & \text{ if } s^\ell_i = + \,, \\
            \vp_-(z_i) = \max(-z_i, 0), & \text{ if } s^\ell_i = - \,.
		 \end{cases}$$
An equivalent definition is the entrywise multiplication $\vp_{s^\ell}(z) = \vpp(s^\ell \odot z)$.  Note that the random signs $s_i^\ell$ are \emph{not} trainable parameters; they are frozen on initialization. 
\mufan{This same symmetrization was first exploited by \citet{allen2019learning,bai2019beyond} to study a quadratic approximation of the network.}
We now present a corresponding limiting theorem for Balanced ResNets.

\begin{thm}\label{prop:balanced_resnets_result}
For a balanced ResNet, the same result as \cref{prop:main} given in \eqref{eq:main} still holds, but with the mean and variance of $G$ given simply by
\begin{equation}\label{eq:balancedResNets}
\e\left[G\right] = -\frac{\beta}{2} +  O\left( \frac{d}{n^2}\right),\hspace{1em} \var\left[G\right] = \beta + O\left( \frac{d}{n^2}\right).
\end{equation}
\end{thm} Balanced ResNets are constructed so that the activation of each neuron is independent of all others due to the random signs $s^\ell$. This eliminates the hypoactivation and variance terms which complicated the analysis of the Vanilla ResNet and necessitated Conjecture \ref{conj:hypoactivation}. Instead, for Balanced ResNets it is straightforward to compute that for any fixed $z,w\in \bR^n$ we have
\begin{equation}\label{eq:balanced_resnets_nice}
\e\left[\norm{\vp_{s^\ell}(z)}^2\right] = \frac{\norm{z}^2}{2},
\var\left[\norm{\vp_{s^\ell}(z)}^2\right] = \sum_{i=1}^n \frac{z_i^4}{4},
\cov\left[ \norm{\vp_{s^\ell}(z)}^2, \norm{\vp_{s^{\ell^\prime}}(w)}^2 \right] = 0.
\end{equation}
Even though the layers $\hat{z}^{\ell}, \hat{z}^{\ell^\prime}$ are correlated, because the activation functions $\vp_{s^\ell}$ and $\vp_{s^{\ell^\prime}}$ are set to be independent on initialization, the correlation between layers does not induce a correlation on which neurons are activated from layer to layer. This explains why there is no hypoactivation and interlayer correlation correction in \cref{prop:balanced_resnets_result} as there is in \cref{prop:main}.

\section{Consequences of Theorems \ref{prop:main} \& \ref{prop:balanced_resnets_result} and Comparison to Infinite-Width Limit}
\subsection{Vanishing and Exploding Norms}\label{sec:consequences}
By the basic fact $\e[\exp(\cN(\mu,\si^2))]=\exp(\mu+\half\si^2)$ it follows from Theorems \ref{prop:main} \& \ref{prop:balanced_resnets_result} that, when the inputs $x$ has $\norm{x}=\sqrt{\nin}$, the mean size scale of any neuron $\zout_i$ is approximately
\begin{align*}
\e\left[ (\zout_i)^2 \right]
&\approx
\begin{cases}
    (\al^2 +\la^2)^{d} \exp\left( 2 c h_\text{total} + \half c^2 I_\text{total} \right) \,, & \text{ for Vanilla ResNets}, \\
    (\al^2 +\la^2)^{d} \,, & \text{ for Balanced ResNets}.
\end{cases}
\end{align*}
(Note that the terms with $\be$ cancel out!) When $\alpha^2 + \lambda^2 = 1$, this is constant for Balanced ResNets. In contrast, Vanilla ResNets have a complicated dependence on the network depth $d$ and width $n$ due to the \emph{hypoactivation} and \emph{correlations} terms. This means the behaviour is $ \exp( C d/n)$, which is somewhat surprising. A more serious issue is the variance which Theorems \ref{prop:main} \& \ref{prop:balanced_resnets_result} predict to be
\begin{align*}
\var\left[ (\zout_i)^2 \right]
&\approx
\begin{cases}
    (\al^2 +\la^2)^{2 d}\left( 3\exp\left( \beta + c^2 I_\text{total} \right) -1 \right) \exp\left( 4c h_\text{total} + c^2 I_\text{total} \right) \,, & \text{ for Vanilla}, \\
    (\al^2 +\la^2)^{2 d}\left( 3\exp\left( \be \right) -1 \right) \,, & \text{ for Balanced}.
\end{cases}
\end{align*}
Since $\beta \approx C d/n$, the term $\exp(\beta)$ represents exponentially larger variance for deep nets compared to shallow ones of the same width. In contrast, the variance predicted by the infinite width limit does not grow with depth for this model when $\al^2+\la^2=1$. This effect means the relative sizes of different network outputs can be widely disparate. Unlike problems with the mean, this issue is harder to resolve. For example, normalization methods that divide all neurons by a constant does nothing to address the large relative disparity between two points. Techniques like batch normalization will be skewed by large outliers. This kind of variance is known to obstruct training \cite{hanin2018start}.
The input-output derivative $\di_{x_i}\zout$ has the same type of behaviour as $\zout$ itself; a simple proof is given in \cref{sec:supp_proofs}. It is expected that the gradient with respect to the weights $\di_{W^\ell_{ij}} \zout$ will also have the same qualitative behaviour \citep{lognormalgrad} although more investigation is needed to understand this theoretically. Exponentially large variance for gradients is a manifestation of the vanishing-and-exploding gradient problem \citep{hanin2018neural}.

Balanced ResNets suffer from this variance problem less because the interlayer correlation term is zero. Since this variance reduction happens at the exponential scale in, the difference can be significant; for networks with $d/n=1$, the contribution is a factor of $\approx e^{5.5} \approx 250$ for Vanilla ResNets vs. $\approx e^{2.5} \approx 10$ for Balanced ResNets. See \cref{fig:out_corr} for a comparison of these theoretically predicted properties vs experiments with finite networks.

\subsection{Correlated Output Neurons}
Since the same random variable $G$ multiplies the entire vector $\zout$, the individual neurons in the output layer are \emph{not} independent. For example, Theorem \ref{prop:main} and \ref{prop:balanced_resnets_result} predict that the squared entries have strictly positive correlation given by
$\corr\left(\left(\zout_i\right)^2, \left(\zout_j\right)^2\right) = (\exp(\si^2)-1)/(3\exp(\si^2)-1)$ for any two neurons $i\neq j$ where $\si^2 = \var(G)$. This tends to $1/3$ as $d/n$ grows. The effect of correlated output neurons persists for Balanced ResNet but is reduced again due to the lower variance. This is very different from the infinite-width limit, which predicts that individual neurons should be independent Gaussians. This prediction of the theorem matches finite networks closely; see \cref{fig:out_corr}.

\begin{figure}[t]
\centering
\hspace{8em} \includegraphics[width=0.95\linewidth]{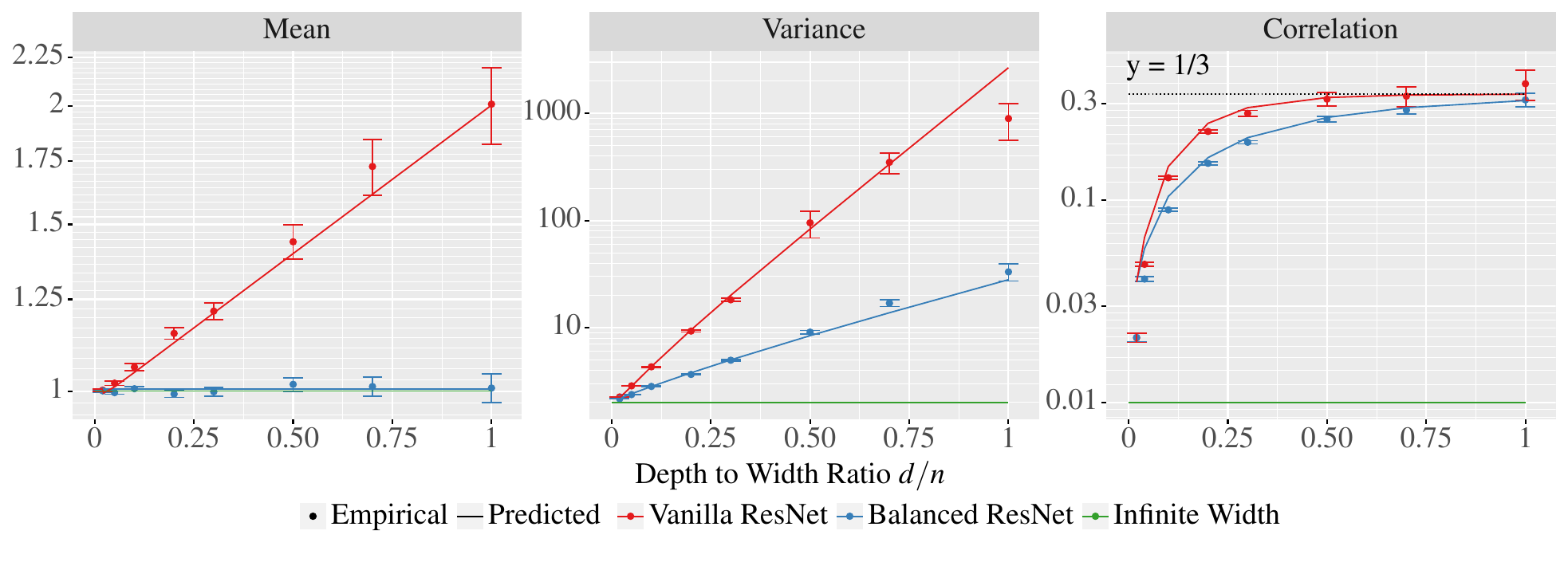}
\caption{
Behaviour of the mean and variance of the typical output $(\zout_i)^2$ and the correlation between $(\zout_i)^2, (\zout_j)^2$ for two different output neurons.
Here $n = 200$, $\alpha = \lambda = 1/\sqrt{2}$ and $d/n$ varies on the $x$-axis.
The infinite width prediction for correlation is zero, but due to plotting on log-scale, we replace zero by $0.01$ for display.
}
\label{fig:out_corr}
\end{figure}

\section{Conjecture \ref{conj:hypoactivation}: Hypoactivation and Layerwise Correlations}\label{sec:hypo_and_corr}
Vanilla ResNets have a subtle asymmetry in the architecture due to the skip connections.
Unlike fully connected networks, the distribution of $\hat{z}^\ell \defequal z^\ell/\norm{z^\ell}$ for $\ell \geq 1$ is \emph{not} exactly uniformly distributed on the unit sphere.
Informally speaking, $\hat{z}^\ell$ can be thought of as a random walk whose variance at each step is proportional to $\norm{\vpp(\hat{z}^\ell)}$. \emph{More} randomness is injected when $\norm{\vpp(\hat{z}^\ell)}$ is \emph{large} and less when it is small.
The net effect is that the walk moves slower when $\norm{\vpp(\hat{z}^\ell)}$ is small, thereby spending more time in those locations.
Consequently, $\|\vpp(\hat{z}^\ell) \|$ is biased toward smaller values.

The size of this effect is limited by entropy; most of the unit sphere $\bS^{n-1}$ has $\norm{\vpp(u)}^2  \approx 1/2$ in the sense that for any $\ep >0$, the measure of the set  $\big\{u:\abs{ \smash{\norm{\vpp(u)}^2 - \half}} > 1/{n^{\half-\epsilon}} \big\}$ vanishes as $n\to\infty$. 

\mufan{To quantify the effect of the bias, we can think of the evolution of $\hat{z}^\ell$, $\ell = 1,2,\ldots $ as a random walk that takes different step sizes at a different points on the sphere. It is reasonable to expect that the  behaviour of this processes will be similar to that of a time changed Brownian motion, which is slowed down at the points where $\hat{z}^\ell$ takes smaller steps. (Proving this comparison precisely is technically difficult since the parameter $n$ simultaneously plays both the role dimension and the step size of the walk.)} Based on \mufan{this heurstic comparison to time changed Brownian motion} and on \mufan{ extensive }Monte Carlo simulations we conjecture that the expected size of the hypoactivation effect is only $O(1/n)$ in expectation; Conjecture \ref{conj:hypoactivation} contains a precise statement.

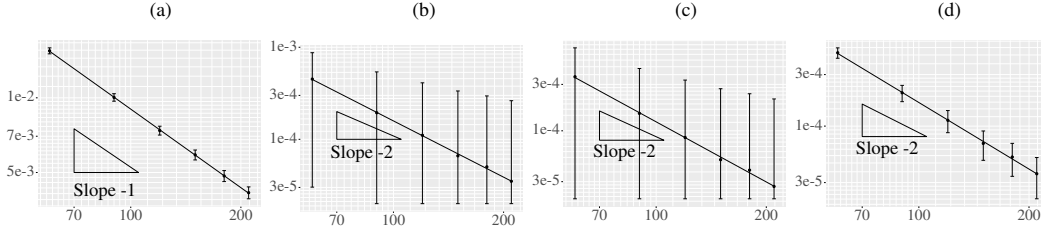
\begin{figure}[t]
\centering
\begin{tikzpicture}[x=0.25\textwidth,y=0.25\textwidth, %
  nplot/.style={anchor=north west}, %
  ncap/.style={anchor=south west}, %
  caption/.style={align=left,font=\normalsize,scale=0.75}, %
  ]

\node[ncap] (figacap) at (0.57,0) [caption] {
(a)
 };
\node[nplot] (figa) at (0,0)  {{\includegraphics[width=0.25\textwidth]{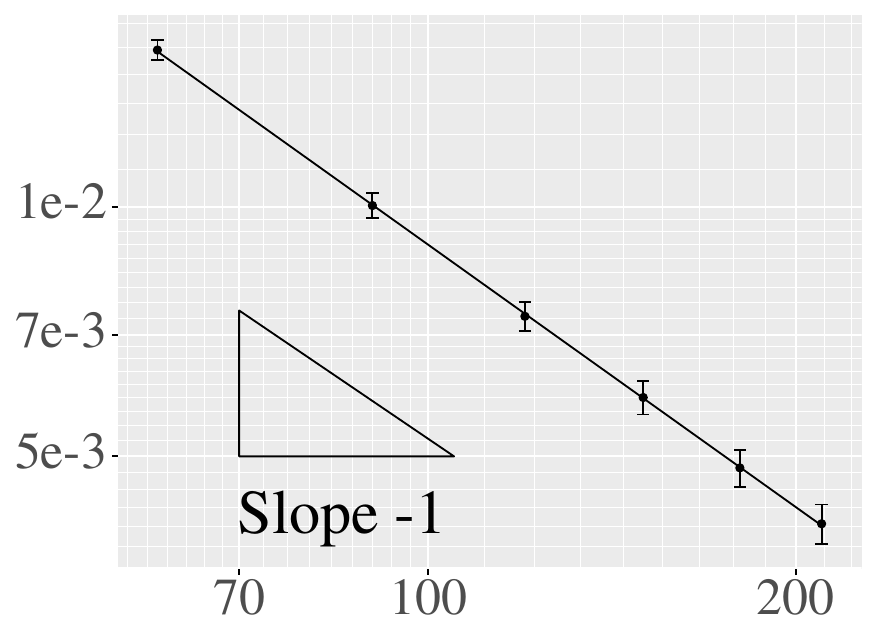}}};  %

\node[ncap] (figbcap) at (1.57,0) [caption] {
(b)
 };
\node[nplot] (figb) at (1,0) {{\includegraphics[width=0.25\textwidth]{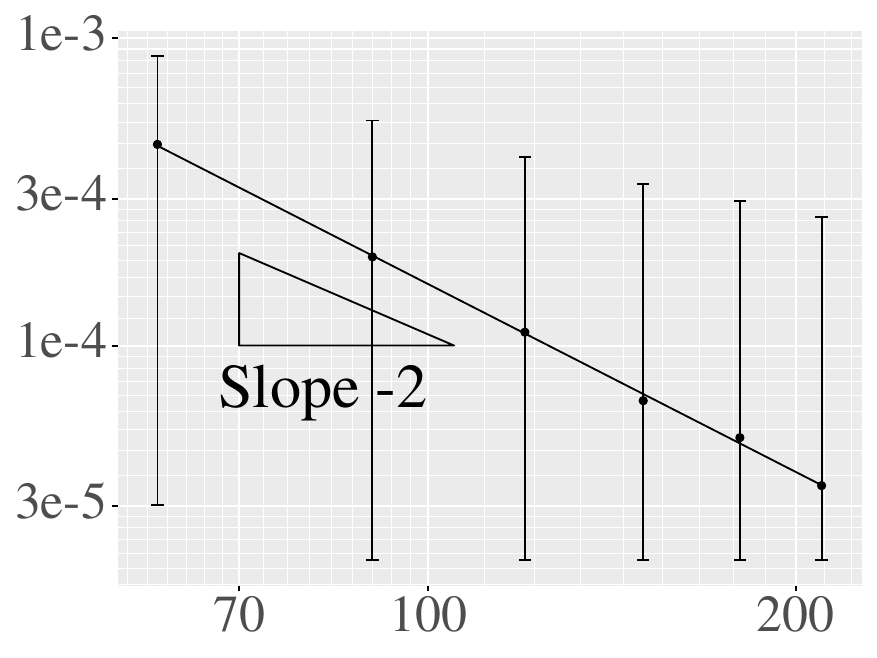}}};

\node[ncap] (figccap) at (2.57,0) [caption] {
(c)
 };
\node[nplot] (figc) at (2,0) {{\includegraphics[width=0.25\textwidth]{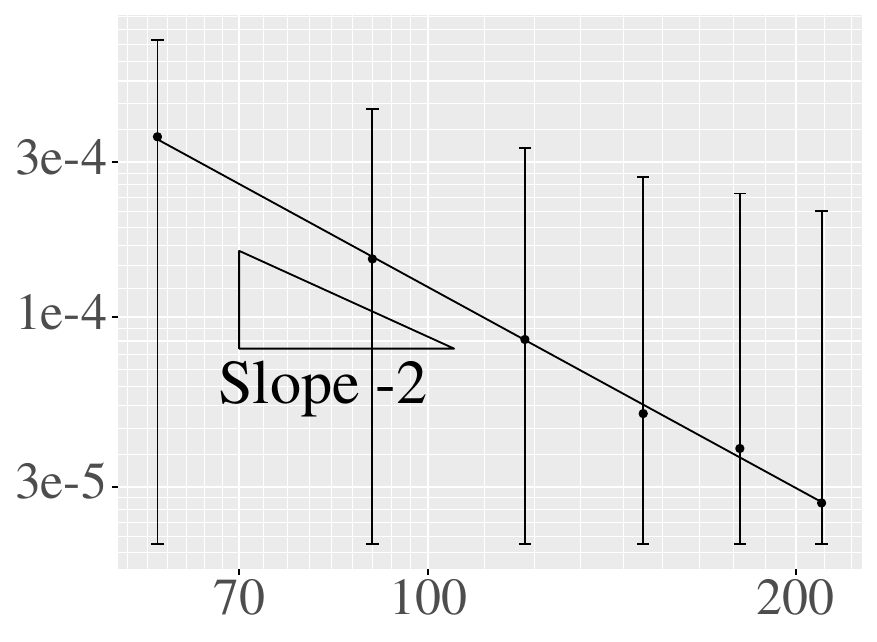}}};

\node[ncap] (figdcap) at (3.57,0) [caption] {
(d)
 };
\node[nplot] (figd) at (3,0) {{\includegraphics[width=0.25\textwidth]{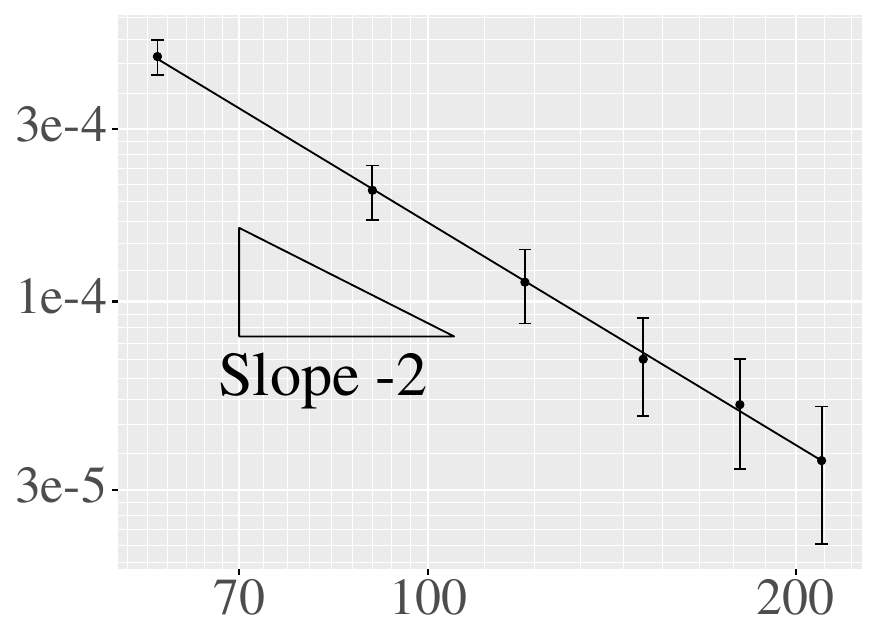}}};

\end{tikzpicture}

\caption{
Monte Carlo evidence for Conjecture \ref{conj:hypoactivation}. $d=200$, $\alpha = \lambda = 1/\sqrt{2}$, and $n$ varies on the $x$-axis.
The plots show the quantities (a) $\left|
\mathbb{E} \| \varphi_+( \hat{z}^\ell ) \|^2
-
\mathbb{E} \| \varphi_+(u) \|^2
\right|$, (b) $\left|
\var \| \varphi_+( \hat{z}^\ell ) \|^2
-
\var \| \varphi_+(u) \|^2
\right|$, (c) $\left|\cov \left( \| \varphi_+( \hat{z}^\ell ) \|^2,
\varphi_+( \hat{z}^{\ell-1} ) \|^2 \right) - C(\theta_1)\right|$, (d) $\left|\cov \left( \| \varphi_+( \hat{z}^\ell ) \|^2,
\varphi_+( \hat{z}^{\ell-2} ) \|^2 \right) - C(\th_2)\right|$. $C(\th_k)$ is the theoretical covariance formula from the term where $\ell^\prime - \ell = k$ in \eqref{eq:h_I_formulas}.
Note also that the absolute error is expected to be $O(n^{-2})$ when the theoretical quantity is $O(n^{-1})$.
Figures
(a) and (b)
verify the conjecture in \eqref{eq:E_conj} and \eqref{eq:var_conj},
and Figures
(c) and (d)
verify \eqref{eq:cov_conj} when $k=1$ and $k=2$ respectively.
For display, we clip the bottom edge of the CI to $2\mathrm{e}{-5}$; otherwise the error bar would go down to $-\infty$ on the log scale.
}
\label{fig:hypo_conj}
\end{figure}
Even if each layer $\hat{z}^\ell$ is marginally close to the uniform distribution on the unit sphere, the directions $\hat{z}^{\ell}$ and $\hat{z}^{\ell+1}$ are not independent because of the skip connections in the network. As above, the exact behaviour is complicated due to fluctuations in the exact number of neurons which are activated in each layer. However, using the idea that $\norm{\vpp(\hat{z}^\ell)}^2  = \half(1+o(1))$, we construct the following approximation.  From \eqref{eq:zl}, we have the approximation ${z^{\ell+1}}/{\norm{z^\ell}} = \al \hat{z}^\ell + \la {g^{\ell+1}}/{\sqrt{n}}\left( 1 + o(1) \right)$
We observe that the norm of RHS is concentrated around $\sqrt{\al^2 + \la^2}$ as $n\to\infty$, so normalizing this to get $\hat{z}^{\ell+1}$ we have
 \begin{equation*}
    \hat{z}^{\ell+1} = \left(\frac{\al}{\sqrt{\al^2+\la^2}} \hat{z}^\ell + \frac{\la}{\sqrt{\al^2 +\la^2}} \frac{g^{\ell+1}}{\sqrt{n}}\right) \left( 1 + o(1) \right).
\end{equation*}
Iterating this gives the same relationship for $\hat{z}^{\ell+k}$ where the first coefficient becomes ${\al^k}/\sqrt{\al^2+\la^2}^k$. As before, based on Monte Carlo simulations, we conjecture that the size of the error is $O(1/n)$ in expectation. We formalize this as a precise statement in \cref{conj:hypoactivation} below.

\begin{conjecture} \label{conj:hypoactivation}
The distribution of the unit vector $\hat{z}^\ell = z^\ell/\norm{z^\ell}$ is approximately uniformly distributed from the unit sphere $u \in \mathbb{S}^{n-1}$
in the precise sense that the following asymptotics hold
\begin{align}
    \label{eq:E_conj} \e\left[ \norm{\vpp(\hat{z}^\ell)}^2 \right] &= \e\left[ \norm{\vpp(u)}^2 \right]\left(1 + O\left(\frac{1}{n}\right)\right) \,, \\
    \label{eq:var_conj}\var\left[ \norm{\vpp(\hat{z}^\ell)}^2 \right] &= \var\left[ \norm{\vpp(u)}^2 \right]\left(1 + O\left(\frac{1}{n}\right)\right) \,, 
\end{align}
\mufan{where the constants in the big $O(\cdot)$ notation are uniform in $\ell$. }
Moreover, for two layers $\ell, \ell^\prime$,
which are $k\geq 1$ layers apart
$\abs{\ell^\prime - \ell} = k$,
the \emph{joint} distribution of
$\hat{z}^\ell, \hat{z}^{\ell^\prime}$
is approximately equal to the joint distribution of
$u,\cos(\th_k)u + {\sin(\th_k)}g/{\sqrt{n}}$
where $g$ is a Gaussian vector with iid $\cN(0,1)$ entries which is independent of
$u$ and $\th$ is such that
$\cos(\th_k)= \al^k/(\al^2+\la^2)^{k/2}$ %
in the sense that the following asymptotics hold
\begin{equation}
    \label{eq:cov_conj} \cov\left[ \norm{\vpp(\hat{z}^\ell)}^2, \norm{\vpp(\hat{z}^{\ell^\prime})}^2  \right] =
    \cov\left[ \norm{\vpp(u)}^2, \norm{\vpp(\cos(\th_k)u + \frac{\sin(\th_k)}{\sqrt{n}}g )}^2  \right]\left(1 + O\left(\frac{1}{n}\right)\right) \,, 
\end{equation}
\mufan{where the constant in the big $O(\cdot)$ notation is uniform in $\ell,\ell^\prime$. }
\end{conjecture}

\mufan{See \cref{fig:hypo_conj} for Monte Carlo simulations empirically verifying the conjecture for a fixed depth $d$, and see \cref{fig:hypo_vs_layer} for verifying the uniformity in layers $\ell$. 
In particular, we observe that in \cref{subfig:hypo_vs_layer}, we can see the effect of hypoactivation converges rapidly to an equilibrium as the layer $\ell$ increases. 
In fact, we can further verify in \cref{subfig:hypo_ar1} that hypoactivation appears to be autoregressive, which implies the convergence is exponentially fast. 
This motivated the uniformity in layers in \cref{conj:hypoactivation}. 
}

\begin{figure}[t]
\centering
\begin{subfigure}[b]{0.55\textwidth}
\centering
\includegraphics[width=\textwidth]{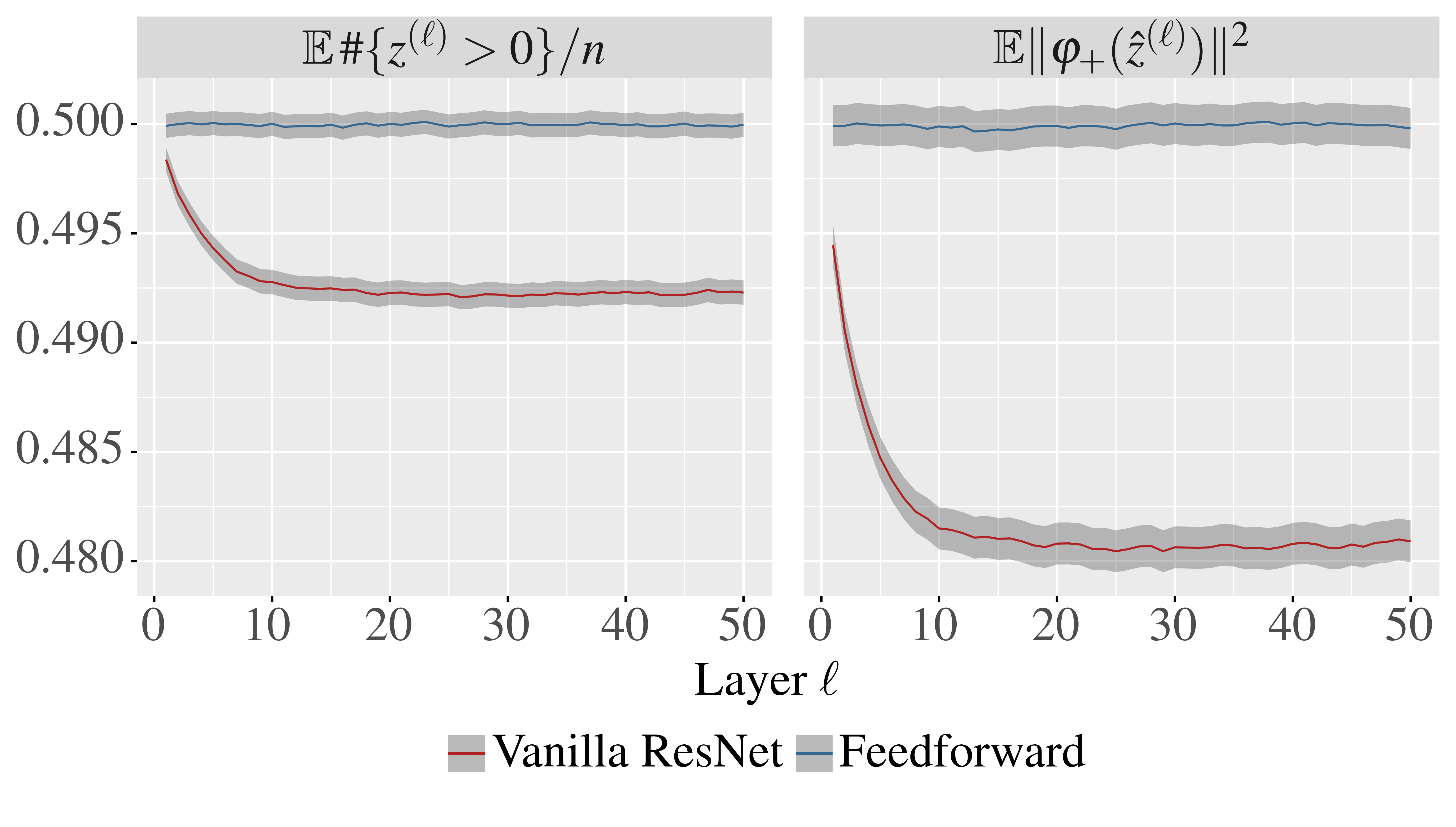}
\caption{
Hypoactivation as a function of layer $\ell$
}
\label{subfig:hypo_vs_layer}
\end{subfigure}
\hfill
\begin{subfigure}[b]{0.35\textwidth}
\centering
\includegraphics[width=\textwidth]{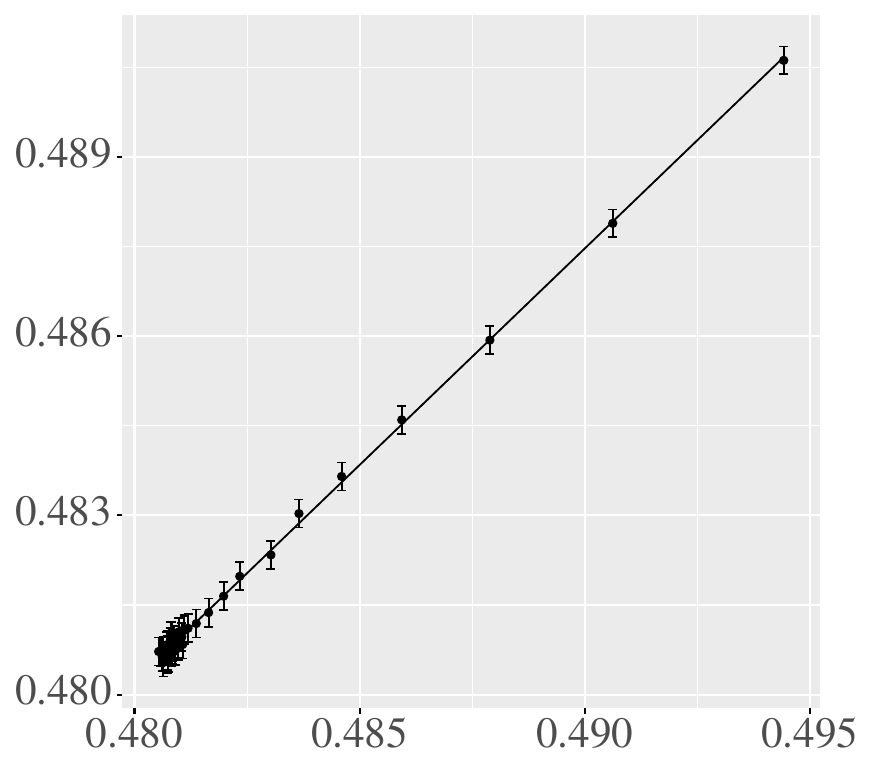}
\caption{
Autoregressive behaviour
$\mathbb{E} \| \varphi_+(\hat{z}^\ell) \|^2$ vs
$\mathbb{E} \| \varphi_+(\hat{z}^{\ell+1}) \|^2$
}
\label{subfig:hypo_ar1}
\end{subfigure}
\caption{
Monte-Carlo simulation for the behaviour of the unit vector $\hat{z}^\ell$ as a function of layer $0\leq \ell \leq d$. Here $n=d=50, \alpha = \lambda = \sqrt{2}^{-1}$.
\cref{subfig:hypo_vs_layer} shows the mean fraction of neurons which are activated, $\e\left[\#\{i:z^\ell_i>0\}\right]/n$ and the norm of the ReLU
$\e\left[ \| \vpp(\hat{z}^\ell) \|^2 \right]$.
The hypoactivation $h_\ell$ is how far this is from $\half$. At layer $\ell=0$, $\hat{z}^\ell$ is uniformly distributed from a unit sphere, but approaches a different steady-state as we go deeper into the network.
\cref{subfig:hypo_ar1} shows evidence that the process
$\e\left[ \| \vpp(\hat{z}^\ell) \|^2 \right]$ seems to be a linear function of the previous layer $\e\left[ \| \vpp(\hat{z}^{\ell-1}) \|^2 \right]$.
\cref{fig:hypo_conj} illustrates the dependence as $n$ varies. }
\label{fig:hypo_vs_layer}
\end{figure}

\section{Proof Ideas for Theorems \ref{prop:main} \& \ref{prop:balanced_resnets_result}}
\label{sec:proof_idea}
A key element of the proof is the following property of Gaussian random matrices. If $W$ which has iid $\cN(0,1)$ entries, then for any vector $x$, we have
\begin{equation} \label{eq:weightmatrix}
    Wx \dequal \norm{x} g,
\end{equation}
where $g$ is a vector whose entries are iid $\cN(0,1)$ random variables. Because of the fully connected first and last layer of the network, \eqref{eq:weightmatrix} implies that
\begin{equation} \label{eq:z0}
    z^0 \dequal \frac{\norm{x}}{\sqrt{\nin}} g, \hspace{1em}
    \zout \dequal \frac{\norm{z^d}}{\sqrt{n}} g^\prime.
\end{equation}
Hence
 $G \defequal  \ln\left( (\norm{z^d}^2/n) \cdot (\norm{x}^2/\nin)^{-1} \cdot \left(\al^2+\la^2\right)^{-d} \right)$ only depends on $n,d,\alpha,\lambda$. (Equivalently, $G$ has the distribution of $\ln(\norm{z^d}^2/n \cdot \left(\al^2+\la^2\right)^{-d})$ when $z^0 = g$.) With this definition, \eqref{eq:z0} also shows $\zout$ is proportional to $\exp(G/2)$, establishing the first part of \cref{prop:main}.

 From this construction, the essence of the proof is to understand the distribution of $\norm{z^d}$ when $z^0 = g$.  To understand $\norm{z^d}$, we look at the ratios $\norm{z^{\ell+1}}/{\norm{z^\ell}}$ layer by layer. By using the homogeneity property of ReLU $\vpp(\abs{c}x) = \abs{c}\vpp(x)$, we can divide $z^{\ell+1}$ from \eqref{eq:defn_resnet} by $\norm{z^\ell}$ to obtain
\begin{equation}\label{eq:zl}
    \frac{z^{\ell+1}}{\norm{z^\ell}} = \al \hat{z}^\ell + \la \sqrt{\frac{2}{n}} W^{\ell+1} \vpp(\hat{z}^\ell) \dequal \al \hat{z}^\ell + \la \sqrt{\frac{2}{n}} \norm{\vpp(\hat{z}^\ell)} g^{\ell+1} ,
\end{equation}
where $g^\ell$ are iid Gaussian vectors with iid $\cN(0,1)$ entries by application of \eqref{eq:weightmatrix}. Hence
\begin{equation*}
    \frac{\norm{z^{\ell+1}}}{\norm{z^\ell}} \dequal \norm{ \al \hat{z}^\ell + \la \sqrt{\frac{2}{n}} \norm{\vpp(\hat{z}^\ell)} g^{\ell+1} } \dequal \norm{ \al \vec{e}_1 + \la \sqrt{\frac{2}{n}} \norm{\vpp(\hat{z}^\ell)} g^{\ell+1} }.
\end{equation*}
The last equality follows by applying an orthogonal transformation $O$ such that $O\hat{z}^\ell = \vec{e}_1 = (1,0,0\ldots,0)^T$ inside the norm, and observing that Gaussian random vectors are invariant under orthogonal transformations $Og^\ell \dequal g^\ell$.
Hence we have the telescoping product for $\norm{z^d}$:
\begin{equation} \label{eq:zd}
 \norm{z^d}= \norm{z^0} \prod_{\ell = 0}^{d-1} \frac{\norm{z^{\ell+1}}^2}{\norm{z^\ell}} \dequal \norm{z^0} \prod_{\ell = 0}^{d-1} \norm{\al \vec{e}_1 + \la \sqrt{\frac{2}{n}}\norm{\vpp(\hat{z}^\ell)} g^{\ell+1} }.
\end{equation}
This shows that $\norm{z^d}$ is a product of $d$ random variables which are dependent on each other only through the terms $\norm{\vpp(\hat{z}^\ell)}$. (Note that $\norm{z^0}$ is independent of $\hat{z}^0$ since $z^0$ is Gaussian.) Since $\norm{\vpp(\hat{z}^\ell)}^2 \approx 1/2$ with typical fluctuations on the scale $1/\sqrt{n}$, therefore the dependence between terms of \eqref{eq:zd} is small. 

Taking the $\ln$ of \eqref{eq:zd} exhibits 
$\ln(\norm{z^d}^2/n)$ as a sum of these weakly correlated random variables. 
\mufan{Here we note that various tail estimates for the same or related quantities have been developed \citep{allen2019convergence,buchanan2021deep}, however these estimates are not precise enough to pinpoint the exact limiting distribution.
In contrast, we are able to derive the exact limiting distribution via a Central Limit Theorem (CLT) for weakly correlated sums \citep{neumann}.}
The proof of \cref{prop:main} is completed by computing the mean, variance and covariance of terms using \cref{conj:hypoactivation}. 
For \cref{prop:balanced_resnets_result}, the final calculation is simplified by \eqref{eq:balanced_resnets_nice} which shows the terms are uncorrelated.
The detailed proof is given in \cref{sec:supp_proofs}.

\section*{Acknowledgement}

We would like to thank 
Blair Bilodeau, 
Gintare Karolina Dziugaite, 
Mahdi Haghifam, 
Yani A. Ioannou, 
James Lucas, 
Jeffrey Negrea, 
Mengye Ren, 
and Ekansh Sharma 
for helpful discussions and draft feedback. 
\mufan{We would also like to thank the anonymous NeurIPS reviewers for insightful feedback. In particular, one identified numerous relations to existing work, and another helped us identify the uniformity requirement in \cref{conj:hypoactivation}}.
ML is supported by Ontario Graduate Scholarship 
and the Vector Institute. 
MN is supported by an NSERC Discovery Grant.
DMR is supported in part by an NSERC Discovery Grant, Ontario Early Researcher Award, and a stipend provided by the Charles Simonyi Endowment.

\printbibliography

\newpage
\appendix
\section{Appendix}

\subsection{Layer dependent coefficients}\label{sec:layer_dependent}
We stated our main result with fixed $\al, \la$, but the result easily extends to the case where $\al,\la$ vary from layer to layer. This allows comparison between our result and the infinite width limits \cite{hayou2021stable, hanin2018start}, where they have $\al=1$ and allow $\la_i$ to varying layer by layer. The statement in that setting is modified as follows.

\begin{prop}
Suppose $\al_i,\la_i$ are sequences such that $\la_i$ is uniformly bounded away from $0$. Define the network by
\begin{equation}
z^{0} = \sqrt{\frac{1}{\nin}}W^0 x, \hspace{1em} z^{\ell} =\al_\ell z^{\ell-1}+\la_\ell \sqrt{\frac{2}{n}} W^\ell {\vpp\left(z^{\ell-1}\right)}\text{ for }1\leq \ell\leq d, \hspace{1em}
\zout = \sqrt{\frac{1}{n}}\Wout z^d
\end{equation}

Consider the limit where \emph{both} the network depth $d \to \infty$ and hidden layer width $n \to \infty$ in such a way that the ratio $\frac{d}{n}$ converges to a constant. In this limit, the distribution of the network output $\zout$ for a given input $x$ is given by
\begin{equation}
    \zout \dequal \frac{\norm{x}}{\sqrt{\nin}} \prod_{i=1}^\ell \sqrt{\al^2_i + \la^2_i } \exp\left(\half G\right) \vec{Z}
\end{equation}
where $\vec{Z}=(Z_1,\ldots,Z_{\nout})\in \bR^{\nout}$ is a Gaussian random vector with iid $\cN(0,1)$ entries and, assuming that Conjecture \ref{conj:hypoactivation} holds, then the random variable $G \in \bR$ converges to a Gaussian random variable in this limit and satisfies
\begin{align} \label{eq:layer-dependent}
 \e\left[G\right] & = -\frac{\beta}{2} +  \sum_{\ell = 1}^d \frac{2\la_\ell^{2}}{\al_\ell^2 + \la_\ell^2} h_\ell + O\left( \frac{d}{n^{2}}\right)\\
\var\left[G\right] & = \beta +\sum_{\stackrel{\ell, \ell^{\prime}=1}{\ell \neq \ell^\prime}}^{d} \frac{\la_\ell^4}{(\al_\ell^2+\la_\ell^2)^2} \frac{\bar{J}_{2}(\th_{\ell,\ell^\prime})-\bar{J}_{2}(\pi-\th_{\ell,\ell^\prime})}{n}+ O\left( \frac{d}{n^{2}}\right),
\end{align}
where $\th_{\ell,\ell^\prime}$ is such that
$\cos\left(\th_{\ell,\ell^\prime}\right) = \prod_{i=\ell}^{\ell^\prime-1} \frac{\al^2_i}{\sqrt{\al^2_i+\la^2_i}}$ and
$\beta = \frac{2}{n}+\frac{1}{n}\sum_{\ell=1}^d \frac{5\la_\ell^4 + 4\al_\ell^2 \la_\ell^2}{(\al_\ell^2+\la_\ell^2)^2}$. 

The corresponding limit theorem for balanced ResNets also holds; the hypoactivation and layer-wise covariance term in \eqref{eq:layer-dependent} vanish.
\end{prop}

The behaviour is a complicated function of the sequences $\al_i,\la_i$. It would be interesting to use these theoretical results to guide the choice of parameters $\al_i,\la_i$ and investigate how this effects training behaviour.

\section{Proof of main results}
\label{sec:supp_proofs}

To simplify the exposition of the proofs, we will assume without loss of generality that $\al^2 + \la^2 = 1$. The general case can be reduced to the case $\al^2 + \la^2 = 1$ by dividing by $\sqrt{\al^2+\la^2}$ in each layer and rescaling the parameters $\al,\la$ to $\frac{\al}{\sqrt{\al^2+\la^2}}$ and $\frac{\la}{\sqrt{\al^2+\la^2}}$.

By the argument of Section \ref{sec:proof_idea}, the proof reduces to showing that the random variable $G = \ln(\norm{z^0}^2/n)+\sum_{\ell=1}^d \ln(X_\ell)$ has the desired asymptotic behaviour where $X_\ell$ is defined to be
\begin{equation} \label{eq:Xell_defn}
    X_\ell = \norm{\al \vec{e}_1 + \la \sqrt{\frac{2}{n}} \norm{\vpp\left( \hat{z}^\ell \right)}g^\ell}^2 \\
    = \al^2 + \la^2 \frac{2}{n}  \norm{\vpp\left( \hat{z}^\ell \right)}^2 \norm{g^\ell}^2 + 2\al \la \sqrt{\frac{2}{n}} \norm{\vpp\left( \hat{z}^\ell \right)} g^\ell_1.
\end{equation}

\subsection{Mean Calculation}
\begin{lem} \label{lm:mean}
$\e[X_\ell] = 1 + 2\la^2 h_\ell$  where $h_\ell$ is the hypoactivation of layer $\ell$.
\end{lem}
\begin{proof}
Taking expectation of both sides of \eqref{eq:Xell_defn}, we have
\begin{align*}
 \e\left[X_\ell\right]    &= \al^2 + \la^2 \frac{2}{n}  \e\left[ \norm{\vpp\left( \hat{z}^\ell \right)}^2 \right] \e\left[ \norm{\vec{g}^\ell}^2\right] + 2\al \la \sqrt{\frac{2}{n}} \e\left[ \norm{\vpp\left( \hat{z}^\ell \right)}\right] \e \left[\vec{g}^\ell_1\right] \\
&= \al^2 + \la^2 \frac{2}{n} \left(  \half + h_\ell\right) n + 0 = 1 + 2\la^2 h_\ell,
\end{align*}
where we have used $\al^2+\la^2=1$ in the last line.
\end{proof}
\begin{cor}
Assume Conjecture \ref{conj:hypoactivation}. Then $h_\ell = O\left(\frac{1}{n}\right)$ and $\e[X_\ell] = 1+O(\frac{1}{n})$.
\end{cor}
\begin{proof} By the conjecture, $h_\ell = \e[\norm{\vpp(\hat{z}^\ell)}^2] - \half = \e[\norm{\vpp(\hat{z}^\ell)}^2]-\e[\norm{\vpp(u)}^2] = O(\frac{1}{n})$.\end{proof}
\subsection{Variance Calculation}
\begin{lem}\label{lm:var_generic}
Let $z \in \bR^{n}$ be any random vector and let $g \in \bR^n$ be an independent Gaussian random vector with iid $\cN(0,1)$ entries. Define
$$X= \norm{\al e_1 + \la \sqrt{\frac{2}{n}} \norm{\vpp(z)} g}^2 = \al^{2}+\la^{2}2\norm{\vpp(z)}^{2}\frac{\norm g^{2}}{n}+\frac{2\al\la}{\sqrt{n}}\norm{\sqrt{2}\vpp(z)}g_{1}.$$
Then
$$\var(X) = \la^{4}\frac{n+2}{n}\var\left[2\norm{\vpp(z)}^{2}\right]+\frac{2\la^{4}}{n}\e\left[2\norm{\vpp(z)}^{2}\right]^{2}+\frac{4\al^{2}\la^{2}}{n}\e\left[2\norm{\vpp(z)}^{2}\right].$$
\end{lem}
\begin{proof}
We will use the decomposition $\var\left(X\right)=\e\left[(X-\al^{2})^{2}\right]+\left(\e[X]-\al^{2}\right)^{2}$ and compute the the two terms individually.

In the second term, since $\e[g_1]=0$ and $\e[\norm{g}^2]=n$, we have $\e[X]-\al^{2}=\la^{2}\e\left[2\norm{\vpp(z)}^{2}\right]$.

To compute $\e[(X-\al^{2})^2]$, notice that $X-\al^{2}=\la^{2}2\norm{\vpp(z)}^{2}\cdot\frac{1}{n}\norm g^{2}+2\al\la\frac{1}{\sqrt{n}}\norm{\sqrt{2}\vpp(z)}g_{1}$
is a sum of two terms. The two terms are uncorrelated since $\cov(\norm{g}^2,g_1) = 0$. Hence, the expectation of the
cross term in $\left(X-\al^{2}\right)^2$ is zero and we can compute
\begin{align*}
\e\left[(X-\al^{2})^{2}\right] & =\la^{4}\e\left[4\norm{\vpp(z)}^{4}\right]\e\left[\frac{\norm{ g}^{4}}{n^{2}}\right]+\frac{4\al^{2}\la^{2}}{n}\e\left[2\norm{\vpp(z)}^{2}\right]\e\left[g_{1}^{2}\right]\\
 & =\la^{4}\e\left[4\norm{\vpp(z)}^{4}\right]\frac{n+2}{n}+\frac{4\al^{2}\la^{2}}{n}\e\left[2\norm{\vpp(z)}^{2}\right].
\end{align*}
We have used the fact about $\ch^{2}_n$ random variables that $\e\left[\norm g^{4}\right]=n(n+2)$. Finally then:
\begin{align*}
\var\left(X\right) & =\e\left[\left(X-\al^{2}\right)^{2}\right]-\la^{4}\e\left[2\norm{\vpp(z)}^{2}\right]^{2}\\
 & =\la^{4}\e\left[4\norm{\vpp(z)}^{4}\right]\frac{n+2}{n}-\la^{4}\e\left[2\norm{\vpp(z)}^{2}\right]^{2}+\frac{4\al^{2}\la^{2}}{n}\e\left[2\norm{\vpp(z)}^{2}\right]\\
 & =\la^{4}\frac{n+2}{n}\var\left[2\norm{\vpp(z)}^{2}\right]+\frac{2\la^{4}}{n}\e\left[2\norm{\vpp(z)}^{2}\right]^{2}+\frac{4\al^{2}\la^{2}}{n}\e\left[2\norm{\vpp(z)}^{2}\right]
\end{align*}
\end{proof}
\begin{lem}\label{lm:var_uniform}
If  $u \in \bS^{n-1}$ is a random vector which is distributed \emph{uniformly} from the unit sphere, then
$$\e\left[2\norm{\vpp(u)}^{2}\right]=1, \hspace{1em} \var\left[2\norm{\vpp(u)}^{2}\right]=\frac{3}{n+2}.$$
\end{lem}
\begin{proof}
This can be calculated directly using properties of the unit sphere, but the proof is complicated by the fact that the entries $u_i,u_j$ are \emph{not} independent. Instead, there is an elementary proof using the following equality in distribution:
\[
\vpp(u)\norm g\dequal\vpp(g),
\]
where $g \in \bR^n$ has iid $\cN(0,1)$ entries and is independent of $u$. This follows because of the fact that $u\dequal\frac{g}{\norm g}$ is independent
of $\norm{g}$, so $u \norm{g} \dequal g$. By also using the fact that $\vp_{+}(\cdot)$ is a positive homogeneous function
$\vp_{+}(\abs cx)=\abs c\vp_{+}(x)$, we see by applying  $\vp_{+}$ to this equality in distribution that $\vp_{+}\left(u\right)\norm g\dequal\vpp(g)$ as desired.

Taking norm and expectation of this equality in distribution gives $\e\left[2\norm{\vpp(u)}^{2}\norm g^{2}\right]=\e\left[2\norm{\vpp(g)}^{2}\right]$.
Since $\norm g$ and $u$ are independent, we can factor and rearrange to obtain
\[
\e\left[2\norm{\vpp(u)}^{2}\right]=\frac{\e\left[2\norm{\vpp(g)}^{2}\right]}{\e\left[\norm g^{2}\right]}.
\]
Since the entries of $g$ are independent
of each other, it is easier to compute using $g$ and this identity instead of using $u$. Moreover, because Gaussian distribution are symmetrically
distributed, we have that $\left\{ g_{1}^{2},\ldots,g_{n}^{2}\right\} $ is independent
of,$\one\left\{ g_{1}>0\right\} ,\ldots,\one\left\{ g_{n}>0\right\} $
Hence:
\begin{equation*}
\e\left[2\norm{\vpp(g)}^{2}\right] =\e\left[2\sum_{i=1}^{n}g_{i}^{2}\one\left\{ g_{i}>0\right\} \right] =2\sum_{i=1}^{n}\e\left[g_{i}^{2}\right]\p\left[g_{i}>0\right] =2n\half=n
\end{equation*}
and so $\e\left[2\norm{\vpp(u)}^{2}\right]=\frac{\e\left[2\norm{\vpp(g)}^{2}\right]}{\e\left[\norm g^{2}\right]}=\frac{n}{n}=1$. Similarly, using $\e\left[4\norm{\vpp(u)}^{4}\right]=\frac{\e\left[4\norm{\vpp(g)}^{4}\right]}{\e\left[\norm g^{4}\right]}$ we now compute by looking at diagonal and off-diagonal terms as follows
\begin{align*}
\e\left[4\norm{\vpp(g)}^{4}\right] & =4\e\left[\sum_{i,j=1}^{n}g_{i}^{2}g_{j}^{2}\one\left\{ g_{i}>0\right\} \one\left\{ g_{j}>0\right\} \right]\\
 & =4\sum_{i=1}^{n}\e\left[g_{i}^{4}\right]\p\left[g_{i}>0\right]+4\sum_{\substack{i,j=1\\
i\neq j
}
}^{n}\e\left[g_{i}^{2}\right]\e\left[g_{j}^{2}\right]\p\left[g_{i}>0\right]\p\left[g_{j}>0\right]\\
 & =4n\cdot3\cdot\half+4n(n-1)\frac{1}{2}\cdot\frac{1}{2} =n^{2}+5n.
\end{align*}
Using $\e\left[\norm g^{4}\right]=n(n+2)$, we finally obtain $\e\left[4\norm{\vpp(u)}^{4}\right]=\frac{n(n+5)}{n(n+2)}$ from which the claimed variance formula follows.
\end{proof}
\begin{cor}
If $u$ is a uniform from the unit sphere $\bS^{n-1}$, and $X= \norm{\al e_1 + \la \sqrt{\frac{2}{n}} \norm{\vpp(u)} g}^2$ as in Lemma \ref{lm:var_generic}, then
$$\var(X) = \frac{5\la^4 + 4\al^2\la^2}{n}.$$
\end{cor}
\begin{proof}
Plug in the result of \cref{lm:var_uniform} into \cref{lm:var_generic}.
\end{proof}
\begin{lem} \label{lm:var}
Assuming Conjecture \ref{conj:hypoactivation} is true, and with $X_\ell$ defined as \eqref{eq:Xell_defn}, we have
$$\var\left(X_\ell\right)=\frac{5\la^4 + 4\al^2\la^2}{n}\left(1+O\left(\frac{1}{n}\right)\right)$$
\end{lem}
\begin{proof}
By the conjecture, we have that $\e\left[ 2\norm{\vpp(\hat{z}^\ell)}^2\right] = \e\left[ 2\norm{\vpp(u)}^2\right]\left(1 + O\left(\frac{1}{n}\right) \right)$
and $\var\left[ 2\norm{\vpp(\hat{z}^\ell)}^2\right] = \var\left[ 2\norm{\vpp(u)}^2\right]\left(1 + O\left(\frac{1}{n}\right) \right)$.
Hence we can compute $\var(X_\ell)$ up to a factor of $\left(1+O\left(\frac{1}{n}\right)\right)$ from Lemma \ref{lm:var_generic} by plugging in the result of Lemma \ref{lm:var_uniform}. 
\end{proof}

\subsection{Uniform distribution on spheres and Gaussian random variables}\label{sec:Gauss_approx} In this section we develop some approximations which are used in the next section. Let $u \in \bS^{n-1} \subset \bR^n$ be a uniform random variable from the unit sphere. Let $Z \sim \cN(0,1) \in \bR $ be a standard Gaussian. The results in this section concern the error rate in the well known approximation for the marginal distribution of the components $u_i$, namely $\sqrt{n}u_i \approx Z$.

\begin{lem}
Let $f:\bR \to \bR$ be any bounded function. Then the marginal distribution of any coordinate $u_i$ satisfies
$$\e\left[f\left(\sqrt{n}u_i\right) \right] = \e\left[f\left(Z\right) \right]+O\left(\frac{1}{n}\right)$$
\end{lem}
\begin{proof}
This is a direct corollary to Theorem 2 of \cite{stam1982limit}, which uses Stirling's formula to show that the total variation distance between the random variables $\sqrt{n}u_i$ and $Z$ is at most $2\left(\sqrt{1+\frac{3}{n-3}} -1\right)=O(1/n)$.
\end{proof}

\begin{lem}
 For $p\in \bN$, the $2p$-th moment of the marginal distribution of any coordinate $u_i$ satisfies
$$\e\left[\left(\sqrt{n}u_i\right)^{2p}\right] = \e\left[Z^{2p}\right]\cdot \left(1\left(1+\frac{2}{n}\right)\cdots\left(1+\frac{2p-2}{n}\right)\right)^{-1} = \e\left[Z^{2p}\right]+O\left(\frac{1}{n}\right)$$ 
\end{lem}
\begin{proof}
As in the proof of \cref{lm:var_generic}, we use the equality in distribution $\norm{g} u \dequal g$ where $g \in \bR^N$ is a vector whose components are iid $\cN(0,1)$ independent of $u$. From this it follows that
$$\e\left[ \norm{g}^{2p} u_i^{2p} \right] = \e\left[g_i^{2p}\right] = \e\left[Z^{2p}\right]$$
The result then follows by using the independence of $\norm{g}$ and $u$, and the formula for the $p$-th moment of $\chi^2_n$ random variable, namely $\norm{g}^{2p} = n(n+2)\cdots(n+2p-2)$.
\end{proof}
\begin{cor}
Let $f:\bR \to \bR$ be any function that satisfies $\abs{f(x)}\leq Ax^{2p} + B$ for some constants $A,B>0$ and exponent $p \in\bN$. Then
$$\e\left[f\left(\sqrt{n}u_i\right) \right] = \e\left[f\left(Z\right) \right]+O\left(\frac{1}{n}\right)$$
\end{cor}
\begin{proof}
The proof is immediate writing the difference in expectation as an integral and then comparing $\intop f(x) \left(\rho_{\sqrt{n}u_i}(x) - \rho_Z(x)\right) dx$ to $\intop \left(Ax^{2p}+B\right) \left(\rho_{\sqrt{n}u_i}(x) - \rho_Z(x)\right) dx$ by the results of the previous two lemmas.
\end{proof}

\subsection{Pairwise covariances}
Define the function $\bar{J}_2 :\bR \to \bR$ by
\begin{equation}\label{eq:J2_def}
\bar{J}_2(\th) \defequal 2\e\left[ \vpp^2(Z) \vpp^2\left(\cos(\th)Z+\sin(\th)W\right) \right],
\end{equation}
where $Z\in \bR,W\in \bR$ are iid $\cN(0,1)$ random variables. In \citet{cho2009kernel} they find an explicit formula for this, namely:
\[
\bar{J}_{2}(\th)=\frac{J_{2}\left(\th\right)}{\pi}=\frac{3\sin(\th)\cos(\th)+\left(\pi-\th\right)\left(1+2\cos^{2}\th\right)}{\pi}
\]

\begin{lem}\label{lm:cov}
Let $u\in\bS^{n-1}$ be a uniform random vectors from the unit sphere and let $g \in \bR^n$ be a Gaussian vector with iid $\cN(0,1)$ entries which is independent of $u$. Then
\begin{equation} \label{eq:cov_theory}
\cov\left(2\norm{\vpp\left(u\right)}^{2},2\norm{\vpp\left(\cos(\th) u+\frac{\sin(\th)}{\sqrt{n}} g\right)}^{2}\right) =  \frac{\bar{J}_2(\theta) - \bar{J}_2(\pi - \theta)}{n}\left(1+O\left(\frac{1}{n}\right)\right).
\end{equation}
\end{lem}
\begin{proof}

By expanding the norms into sums, $\norm{x}^{2} = \sum_{i=1}^n x_i^2$, we can compute the covariance by summing over all \emph{pairs} of coordinates $u_i,\cos(\th)u_j+\sin(\th)n^{-\half}g_j$. There are two types of terms to consider. (Note: we use the notation $\vpp^2(x)=(\vpp(x))^2$.)

\underline{Diagonal Terms}: $\cov\left(\vpp^2(u_i),\vpp^2(\cos(\th)u_i +\sin(\th)n^{-\half}g_i)\right)$ for $1\leq i \leq n$
We first use the positive homogeneity of $\vpp$ to extract a factor of $\sqrt{n}$ from both terms
$$ \cov\left( \vpp^2(u_i),\vpp^2(\cos(\th) u_i + \frac{\sin(\th)}{\sqrt{n}} g_i)\right) = \frac{1}{n^2} \cov\left( \vpp^2(\sqrt{n}u_i),\vpp^2(\cos(\th) \sqrt{n}u_i + \sin(\th) g_i)\right) $$
We now use the approximation of $\sqrt{n}u_i$ by $Z \sim \cN(0,1)$ as in Section \ref{sec:Gauss_approx} to obtain
\begin{align*}
& \frac{1}{n^2} \cov\left( \vpp^2(\sqrt{n}u_i),\vpp^2(\cos(\th) \sqrt{n}u_i + \sin(\th) g_i)\right)\\
=& \frac{1}{n^2} \cov\left(\vpp^2\left(Z\right),\vpp^2\left(\cos(\th) Z + \sin(\th) g_i\right) \right)\opon \\
=& \frac{ \frac{1}{2}\bar{J}_2(\theta) - \frac{1}{4}}{n^2} \opon,
\end{align*}
where we have used the definition of $\bar{J}_2(\theta)$ from \eqref{eq:J2_def}.

\underline{Off diagonal terms} $\cov\left(\vpp^2(u_i),\vpp^2(\cos(\th)u_j +\sin(\th)n^{-\half}g_j)\right)$ for $1\leq i \leq n$,$1\leq j\leq n$,$i\neq j$

We compute the expectation
$$\e\left[ \vpp^2(u_i)^2 \vpp^2(\cos(\th) u_j + \sin(\th)n^{-\half} g_j ) \right]$$
by first conditioning on $u_j$. Conditioned on $u_j$, the distribution of $u_i$ is $u_i \dequal \sqrt{1-u_j^2} \tilde{u_i}$, where $\tilde{u}$ is independent of $u$ and is drawn uniformly from the unit sphere whose dimension is one smaller than that of $u$, namely $\tilde{u} \in \bS^{n-2}$. Since the $\vpp$ is positive homogeneous, we can factor $\sqrt{1-u_j^2}$ out to get:
\begin{align*}
    \e\left[ \vpp^2(u_i) \vpp^2(\cos(\th) u_j + \sin(\th) n^{-\half} g_j ) \right] &=\e\left[ \vpp^2(\tilde{u}_i)^2 (1-u_j^2) \vpp^2(\cos(\th) u_j + \sin(\th) n^{-\half} g_j ) \right] \\
    &=\e\left[ \vpp^2(\tilde{u}_i)^2\right] \e\left[(1-u_j^2) \vpp^2(\cos(\th) u_j + \sin(\th) n^{-\half} g_j ) \right] \\
    &= \frac{1}{2(n-1)} \e\left[ (1-u_j^2) \vpp^2(\cos(\th) u_j + \sin(\th) n^{-\half} g_j ) \right]
\end{align*}
As in the calculation for the diagonal term, we now again use the approximation that $\sqrt{n}u_j$ is approximately marginally distributed like $Z\sim\cN(0,1)$ and the positive homogoneity of $\vpp$ to obtain
\begin{align*}
    &\e\left[ \vpp^2(u_i) \vpp^2(\cos(\th) u_j + \sin(\th) n^{-\half} g_j ) \right]\\ =& \frac{1}{2(n-1)}\left( \frac{1}{2n} -  \e\left[ Z^2 \vpp^2(\alpha Z + \lambda g_j ) \right] \right) \opon. 
\end{align*}
Renaming $g_j$ to $W$ to match the notation of \eqref{eq:J2_def}, we finally obtain
\begin{align*}
&\cov\left( \vpp^2(u_i)^2, \vpp^2(\cos(\th) u_j + \sin(\th) n^{-\half} g_j ) \right) \\
=& \e\left[ \vpp^2(u_i) \vpp^2(\cos(\th) u_j + \sin(\th) n^{-\half} g_j ) \right] - \frac{1}{4n^2}\opon \\
=& \left(\frac{1}{4n^2(n-1)}-\frac{1}{2n^2(n-1)} \e\left[ Z^2 \vpp^2(\cos(\th) Z + \sin(\th) W ) \right]\right) \opon
\end{align*}
Summing the diagonal and off-diagonal terms, we find the total covariance is:
\begin{align*}
    &\cov\left(2\norm{\vpp\left(u\right)}^{2},2\norm{\vpp\left(\cos(\th) u+\sin(\th) n^{-\half} g\right)}^{2}\right)\\
    =& 4 n \times \left( \text{Diagonal term contribution} \right) +  4n(n-1) \times \left( \text{Off Diagonal term contribution} \right) \\
    =& \left(\frac{2 \bar{J_2}(\theta) - 1}{n} + \frac{1-2\e\left[ Z^2 \vpp^2(\alpha Z + \lambda W ) \right]}{n}\right)\opon \\
    =& \left(\frac{2 \bar{J_2}(\theta) -2\e\left[ Z^2 \vpp^2(\alpha Z + \lambda W ) \right]}{n} \right)\opon
\end{align*}

Finally we notice the identity
\begin{align*}
    &2\bar{J_2}(\theta) -2\e\left[ Z^2 \vpp^2(\cos(\th) Z + \sin(\th) W ) \right] \\
    =& 4\e\left[ \vpp^2(Z) \vpp^2(\cos(\th) Z + \sin(\th) W ) \right] -2\e\left[ Z^2 \vpp^2(\cos(\th) Z + \sin(\th) W ) \right] \\
    =& \e\left[\right( 2\vpp^2(Z) - Z^2\left) 2\vpp^2(\cos(\th) Z + \sin(\th) W ) \right]\\
    =& \e\left[\left( \vpp^2(Z) - \vpp^2(-Z)\right) 2\vpp^2(\cos(\th) Z + \sin(\th) W ) \right] \\
    =&  2\e\left[\vpp^2(Z) \vpp^2(\cos(\th) Z + \sin(\th) W ) \right] - 2\e\left[\vpp^2(-Z) \vpp^2(\cos(\th) Z + \sin(\th) W ) \right] \\
    =& 2\e\left[ \vpp^2(Z) \vpp^2(\cos(\th) Z + \sin(\th) W ) \right] - 2\e\left[ \vpp^2(Z) \vpp^2(-\cos(\th) Z + \sin(\th) W ) \right] \\
    =& \bar{J}_2(\theta) - \bar{J}_2(\pi - \theta),
\end{align*}
which gives the claimed formula for the covariance.
\end{proof}
\begin{cor}
Assume Conjecture \ref{conj:hypoactivation} is true. Then
\begin{equation}
\cov\left(2 \norm{ \vpp^2\left(\hat{z}^\ell\right)},2 \norm{\vpp^2\left(\hat{z}^{\ell^\prime}\right)}^{2}\right)  = \frac{\bar{J}_{2}\left(\th_{\abs{\ell - \ell^\prime}}\right)-\bar{J}_{2}\left(\pi - \th_{\abs{\ell - \ell^\prime}}\right)}{n}\opon,
\end{equation}
where $\th_{k}$ is the angle such that $\th_k=\cos^{-1}(\al^k)$, 
\mufan{and the constant in the big $O(\cdot)$ notation is uniform in $\ell,\ell^\prime$. }
\end{cor}
\begin{proof}
This follows immediately from the approximation for the covariance in \cref{conj:hypoactivation} and \cref{lm:cov}.
\end{proof}
\begin{rem}
By the series expansion for $\arccos(x)=\frac{\pi}{2} - x +O(x^3)$ as $x\to 0$, it follows that $\bar{J}_2(\th_k) - \bar{J
}_2(\pi-\th_k) = 8\al^{k}/\pi + O\left(\al^{2k}\right)$ as $k\to\infty$. This exponential decay in $k$ explains why the total covariance correction $I_\text{total}$ remains $O(d/n)$ even as as $d\to\infty$.
\end{rem}
\begin{lem}
Assume Conjecture \ref{conj:hypoactivation} is true. Then
\begin{equation}
\cov\left(X_{\ell},X_{\ell^\prime}\right) = \la^{4} \frac{\bar{J}_{2}\left(\th_{\abs{\ell - \ell^\prime}}\right)-\bar{J}_{2}\left(\pi - \th_{\abs{\ell - \ell^\prime}}\right)}{n}\opon,
\end{equation}
where $\th_{k}$ is the angle such that $\th_k=\cos^{-1}(\al^k)$, 
\mufan{and the constant in the big $O(\cdot)$ notation is uniform in $\ell, \ell^\prime$. }
\end{lem}

\begin{proof}
From the definition of $X_\ell$ from \eqref{eq:Xell_defn} as a sum of three terms, the covariance can be written as a sum over pairwise combinations of the terms. The terms involving $g_1^\ell,g_1^{\ell^\prime}$ are mean $0$ and independent from layer to layer, so these terms has no contribution to the covariance. We remain with
\begin{equation*}
    \cov\left( X_\ell, X_{\ell^\prime} \right) = \cov\left(2\la^2 \norm{ \vpp^2\left(\hat{z}^\ell\right)}\frac{\norm{g^\ell}^{2}}{n},2\la^2 \norm{\vpp^2\left(\hat{z}^{\ell^\prime}\right)}^{2}\frac{\norm{g^{\ell^\prime}}^2}{n}\right).
\end{equation*}
The result then follows since $\norm{g^\ell}^2 = n$, $g^\ell$ is independent of $g^{\ell^\prime}$, applying the approximation given in Conjecture \ref{conj:hypoactivation} and the result of Lemma \ref{lm:cov}.
\end{proof}

\subsection{Central Limit Theorem}
\begin{lem}\label{lm:bounds}
We have the following asymptotics for the centered moments of $X_\ell - \e[X_\ell]$:
$$\e\left[ \left( X_\ell-\e[X_\ell] \right)^3 \right] = O\left(\frac{1}{n^2}\right), \e\left[ \left( X_\ell-\e[X_\ell] \right)^{2m} \right] = O\left(\frac{1}{n^{m}}\right) \text{ for } m\geq 2 \,, $$
\mufan{where the constant in the big $O(\cdot)$ notation is uniform in $\ell$. }
\end{lem}
\begin{proof}

For convenience of notation, let $c=\sqrt{2}\norm{\vpp(\hat{z}^\ell)}$.
Conditionally on the value of the random value $c$,
$X_{\ell}=\norm{\al\vec{e}_{1}+c\la n^{\half}g}^{2} = \frac{c^{2}\la^{2}}{n}\norm{\frac{\al\sqrt{n}}{c\la}\vec{e}_{1}+g}^{2}$
is a multiple of a non-central $\ch_{k}^{2}\left(\frac{\al^{2}n}{c^{2}\la^{2}}\right)$
random variable.
It is a basic fact about non-central chi squared random variables that if
$Y\sim\chi_{k}^{2}(\mu)$,
then $$\e\left[\left(Y-\e[Y]\right)^{3}\right]=8(3\mu+k)$$
Hence $$\e\left[\left(X_{\ell}-\e[X_{\ell}]\right)^{3}\given c\right]=\left(\frac{c^{2}\la^{2}}{n}\right)^{3}8\left(3\frac{\al^{2}n}{c^{2}\la^{2}}+n\right)=\frac{24c^{4}\la^{4}\al^{2}+8c^{6}\la^{6}}{n^{2}}$$
Using the fact that $0\leq c\leq\sqrt{2}$ almost surely gives $\e\left[\left(X_{\ell}-\e[X_{\ell}]\right)^{3}\right]=O\left(\frac{1}{n^{2}}\right)$ as desired. A similar computation can be carried out to see the bound for the $2m$-th central moment by using the $2m$-th moment of a non-central $\chi^2_n(\mu)$ distribution. This can be seen using the formula for the $m$-th cumulant:
if $Y\sim\chi_{k}^{2}(\mu)$, then $K_m = 2^{m-1}(m-1)!(m\mu+k)$, from which it follows by the formula to convert from cumulants to central moments.
\end{proof}

\begin{lem}\label{lm:CLT_sum}
Without loss of generality, assume that $\al^2 + \la^2=1$. Let $S = \sum_{\ell=1}^d \ln{X_\ell}$ then
\begin{align}\label{eq:ES}
\e\left[S\right] & = -\frac{d}{2n}\cdot\left(5\la^4 + 4\al^2\la^2\right) + 2\la^{2}\sum_{\ell = 1}^d h_\ell + O\left( \frac{d}{n^{2}}\right) \,, \\
\label{eq:varS}
\var\left[S\right] &= \frac{d}{n}\cdot\left(5\la^4 + 4\al^2\la^2\right)
	+ \la^4 \sum_{1\leq \ell \neq \ell^\prime \leq d}
	\frac{\bar{J}_{2}(\th_{|\ell^\prime - \ell|})-\bar{J}_{2}(\pi-\th_{|\ell^\prime - \ell|})}{n}+ O\left( \frac{d}{n^{2}}\right),
\end{align}
\mufan{where the constant in the big $O(\cdot)$ notation is uniform in $d$.} 
Moreover, $S$ is asymptotically Gaussian in the infinite width and depth limit.
\end{lem}

\begin{proof}
By Lemmas \ref{lm:mean} and \ref{lm:var_uniform},
we see that $\e[X_\ell] = 1+O(\frac{1}{n})$ and $\var[X_\ell] = O(\frac{1}{n})$.
Moreover, conditionally on $\norm{\vpp(\hat{z}^\ell)}$, $X_\ell$ has a non-central $\chi^2_n$ distribution 
Hence, by Chebyshev's inequality, we know that for any $\ep >0$, $\frac{X_{\ell} -  \e[X_\ell]}{\e\left[X_{\ell}\right]} = O(\frac{1}{n^{\half-\epsilon}})$ with probability at least $1-O(n^\epsilon)$. On this event, we can hence take the Taylor series expansion of $\ln(x)$ around $x=1$ to obtain the following 
\begin{align*}
\ln(X_{\ell})  =&\ln\left(\e\left[X_{\ell}\right]\right)+\ln\left(1+\left(\frac{X_{\ell} - \e[X_\ell]}{\e\left[X^{\ell}\right]}\right)\right)\\
  =& \ln\left(\e\left[X^{\ell}\right]\right)+\left(\frac{X^{\ell} - \e[X^\ell]}{\e\left[X^{\ell}\right]}\right)-\half\left(\frac{X^{\ell} - \e[X^\ell]}{\e\left[X^{\ell}\right]}\right)^{2}\\
 &+\frac{1}{3}\left(\frac{X^{\ell} - \e[X^\ell]}{\e\left[X^{\ell}\right]}\right)^{3}-\frac{1}{4}\left(\frac{X^{\ell} - \e[X^\ell]}{\e\left[X^{\ell}\right]}\right)^{4}
  + O\left(\frac{1}{n^{\frac{5}{2}-5\epsilon}}\right)
\end{align*}
Using Lemmas \ref{lm:mean}, \ref{lm:var}, \ref{lm:cov}, the bounds from Lemma \ref{lm:bounds}, and the fact that $\ln(X_\ell)$ has finite moments, we can take the expectation and variance of this to obtain
\begin{align*}
    \e\left[ \ln(X_\ell) \right] &= 1+2\la^2h_\ell -\half\frac{5\la^4+4\la^2\al^2}{n} + O\left(\frac{1}{n^{2}}\right),\\
    \var\left[ \ln(X_\ell) \right] &=  \frac{5\la^4+4\la^2\al^2}{n} + O\left(\frac{1}{n^{2}}\right),\\
    \cov\left( \ln(X_\ell),\ln(X_{\ell^\prime}) \right) &= \la^{4} \frac{\bar{J}_{2}\left(\th_{\abs{\ell - \ell^\prime}}\right)-\bar{J}_{2}\left(\pi - \th_{\abs{\ell - \ell^\prime}}\right)}{n} + O\left(\frac{1}{n^{2}}\right),
\end{align*}
from which the desired mean and variance formula for $S$ follows.

The fact that $S$ is asymptotically Gaussian follows by application of the Central Limit Theorem for weakly dependent triangular arrays \citep[][Thm.~2.1]{neumann}
using the bounding sequence $\la^4 \left( \bar{J_2}(\th_r) - \bar{J_2}(\pi-\th_r) \right) $. Note that the sequence $\hat{z}^0, \hat{z}^1,\hat{z}^2,\ldots $ is a Markov chain, which simplifies the  verification of the covariance condition in this central limit theorem. The fact that the sequence satisfies the Feller condition is also clear because of the independent Gaussian random vectors that appear in the definition of $X_\ell$.
\end{proof}

\begin{proof}[Proof of \cref{prop:main}]
Recall that $G = \ln{\norm{z^0}}^2/n + \sum_{\ell=1}^d X_\ell$. By the telescoping product \eqref{eq:zd}, we have that
\begin{equation}
    \ln\left(\frac{\norm{z^d}^2}{n}\right) = \ln\left(\frac{\norm{z^0}^2}{n}\right) + \sum_{\ell=1}^d \ln{X_\ell}
\end{equation}
Also by Section \ref{sec:proof_idea}, $\ln \frac{\norm{z^0}^2}{n} \sim \ln\frac{\norm{x}^2}{\nin} + \ln\frac{\chi^2_n}{n}$ can be written in terms of a chi-squared distribution with $n$ degrees of freedom. By standard facts about $\chi^2(n)$ random variables, we have the a central limit theorem as $n\to \infty$ namely
$\ln\frac{\norm{z^0}^2}{n} \to \cN\left(\ln\frac{\norm{x}^2}{\nin} - \frac{1}{n}. \frac{2}{n}\right) $. We also note that $\norm{z^0}$ is independent of $X_\ell$ for all $\ell$. By Lemma \ref{lm:CLT_sum} the sum
$\sum_{\ell=1}^d\ln X_\ell$ converges to an independent Gaussian random variable. The result follows from the fact that a sum of two independent Gaussians is again Gaussian.
\end{proof}

\begin{proof}[Proof of \cref{prop:balanced_resnets_result} ]
The proof is analogous to the proof of \cref{prop:main}; the calculation for the activation of the layers is \mufan{simplified} by the fact that which neurons activated in each layer are uncorrelated by  \eqref{eq:balanced_resnets_nice}.
\end{proof}

\subsection{Input-Output Gradient for Balanced ResNets}

\begin{prop} For any input $x$, the output of a \emph{Balanced ResNet} can be written as
$$\zout = M(x) x$$
where $M(x)$ is an $\nin \times \nout$ matrix that depends on $x$ and the random network weights. Moreover, the marginal distribution of $M(x)$ is statistically \emph{independent} of $x$. Finally, the matrix $M(x)$ has the property that for almost every $x$, there is an open neighbourhood $\cA(x) \subset \bR^n$ containing $x$ so that $M(x)$ is constant on $\cA(x)$.
\end{prop}
\begin{proof}
First note that because the Gaussian weights are continuous random variables, the event that any of the entries of $z^\ell_i$ are exactly 0 is a probability 0 event. Hence, for almost every input $x$, it makes sense to consider the derivative of the ReLU function $\vpp^\prime$ evaluated at the neuron values. By writing the action of a ReLU function on a vector $\vp(z)$ as a matrix multiplication by a diagonal vector of $1$'s and $0$'s as $\vp(z)=\text{diag }\left(\vp^{\prime}(z)\right)z$, the update rule for ResNets can be written as matrix multiplication, namely
\begin{align} \label{eq:balanced_resnet_zell}
    z^{\ell}
		&=\al z^{\ell-1}+\la \sqrt{\frac{2}{n}}{W^{\ell}}{\vp_{s^\ell}\left(z^{\ell-1}\right)}\\
		&=\al z^{\ell-1}+\la \sqrt{\frac{2}{n}} W^{\ell}\diag\left(\vp_{s^\ell}^{\prime}(z^{\ell-1})\right)z^{\ell-1}\\
		&=\left(\al I+\la \sqrt{\frac{2}{n}}W^{\ell}\diag\left(\vp_{s^\ell}^{\prime}(z^{\ell-1})\right)\right)z^{\ell-1}
\end{align}
For the Balanced ResNet, the sign $s^\ell$ used is independent from layer to layer and neuron to neuron. Therefore the derivative, $\vp_{s^\ell}^\prime$ at any input is equally likely to be $0$ or $1$ independent of everything else (i.e. no matter if the input $z_i$ is positive or negative, $\vp_{s^\ell}^\prime(z_i) = 1$ exactly half the time, and $\vp_{s^\ell}^\prime(z_i) = 0$ exactly half the time.) Hence, from \eqref{eq:balanced_resnet_zell}, we see that the output $\zout$ is equal in distribution to
\begin{equation} \label{eq:bernoullis}
    \zout \dequal \sqrt{\frac{1}{n}}W^\text{out} \prod_{\ell=1}^d \left(\al I + \la \sqrt{\frac{2}{n}} W^\ell \diag\left(B^\ell_i\right) \right) \sqrt{\frac{1}{\nin}}W^0 x,
\end{equation}
where $B_i^\ell$ are fair $\{0,1\}$ Bernoulli random variables independent of the weights $W^\ell$. Equation \eqref{eq:bernoullis} shows the marginal distribution of $M(x)$ does not depend on $x$.  Finally note that Bernoulli random variables depend only on the sign of the intermediate neurons $z^\ell$ at the input $x$ and the signs $s^\ell$. Therefore, if we find a neighbourhood $\cA(x)$ of $x$ such that none of the neurons change sign from positive to negative within the region $x$, the matrix will remain constant. This is always possible as long as none of $z^\ell_i$ are exactly $0$ since they are continuous functions of $x$. But $z^\ell_i =0$ is a probability zero event, so we can find such a neighbourhood $\cA(x)$ for almost every $x$ as desired.
\end{proof}
\begin{cor}
The derivative of $\zout$ with respect to any input $x_i$ has the distribution
$$ \frac{\di}{\di x_i} \zout \dequal M(x) e_i $$
\end{cor}
\begin{proof}
This follows immediately since $M(x)$ is constant on the neighbourhood $\cA(x)$.
\end{proof}
\begin{cor}
$\frac{\di}{\di x_i} \zout$ has the distribution as the output $\zout$ at any input $x$ with $\norm{x} = 1$.
\end{cor}
\begin{proof}
By the previous results, both are equal in distribution to $Mu$ for any unit vector $u$, where $M$ is the distribution of the random matrix in \eqref{eq:bernoullis}.
\end{proof}

\section{Experiments}
\label{sec:experiment}

Throughout the paper, the Monte Carlo simulations were computed on a single NVIDIA Titan-XP GPU. The main tools used in the neural network simulations are the JAX library \citep{jax2018github} (Apache 2.0 License)
and the PyTorch library \citep{pytorch2019} (BSD 3-Clause License).
Furthermore, the Python libraries
numpy \citep{harris2020array} (BSD 3-Clause License),
plotnine \citep{plotnine} (based on ggplot2 \citep{ggplot2}, GNU GPLv2 License),
and pandas \citep{pandas} (BSD 3-Clause License)
tremendously helpful.
We used Python version 3.6.8 from Anaconda 3 \citep{anaconda} (3-clause BSD License) and Jupyter notebook \citep{Kluyver:2016aa} (3-Clause BSD License).

\subsection{Vanilla ResNet and Balanced ResNets: MNIST and CIFAR-10 Experiments}

\begin{figure}[h] %
\centering
\begin{subfigure}[b]{0.495\textwidth}
\centering
\includegraphics[width=\textwidth]{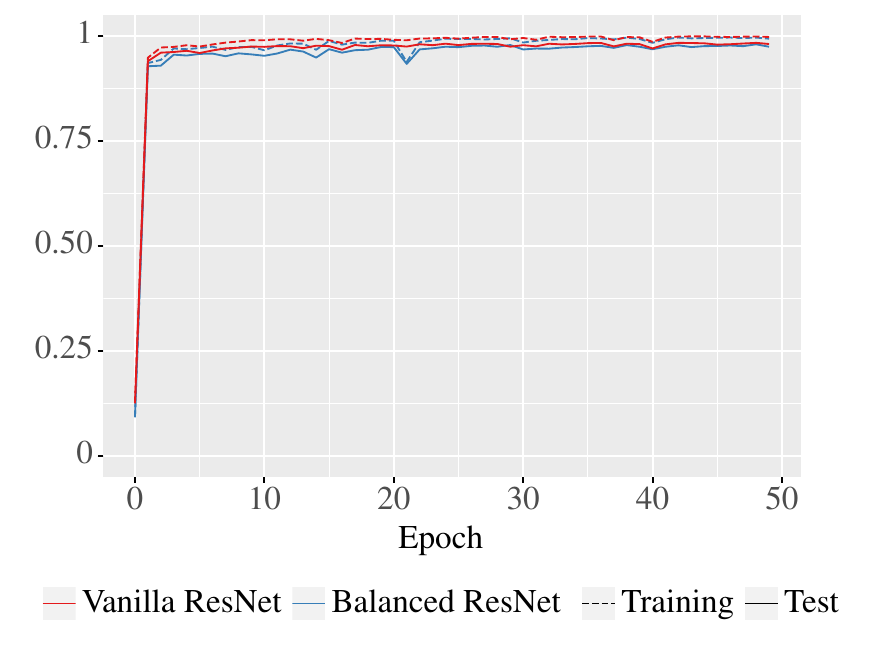}
\caption{MNIST}
\label{subfig:MNIST}
\end{subfigure}
\begin{subfigure}[b]{0.495\textwidth}
\centering
\includegraphics[width=\textwidth]{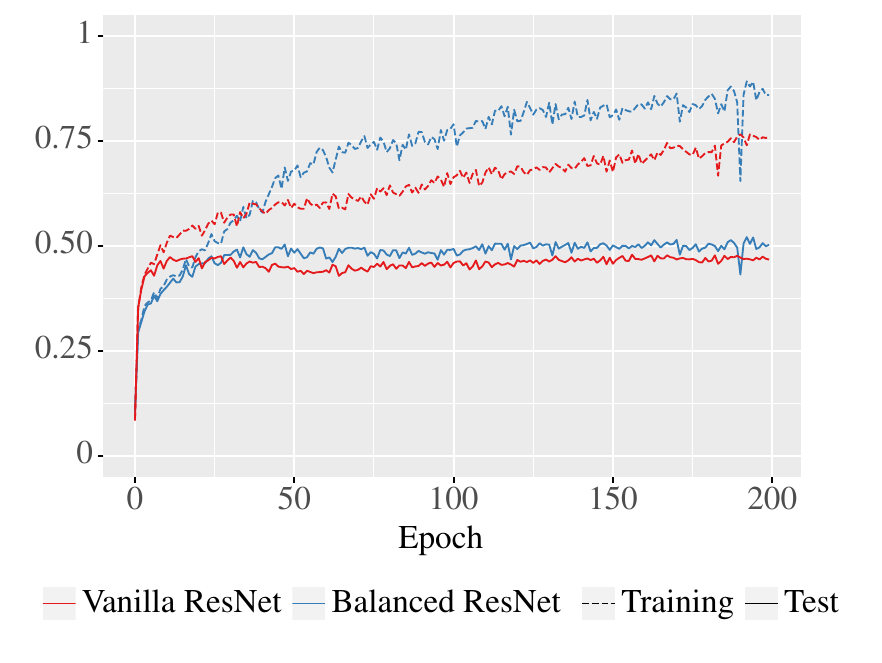}
\caption{CIFAR-10}
\label{subfig:CIFAR10}
\end{subfigure}
\caption{Fully connected Vanilla ResNet and Balanced ResNet
accuracies on MNIST \& CIFAR-10 using ADAM optimization \citep{kingma2015adam} 
with learning rate $0.01$, 
$b_1=0.9, b_2=0.999, \epsilon=1\mathrm{e}{-8}$, 
and batch size $128$. Both networks use hidden layer sizes $n=1000, d=30$ and $\al=1/\sqrt{2},\la=1\sqrt{2}$}
\label{fig:resnet-fc}
\end{figure}

These experiments, displayed in \cref{fig:resnet-fc}, investigate the difference between \emph{training performance} of the Vanilla ResNet from \eqref{eq:defn_resnet} and the Balanced ResNet defined in \cref{subsec:balanced_resnets}. This is beyond the theory proven in our work which concerns statistical properties of the network on initialization. Both of these architectures are fully-connected networks with skip connections between layers. We observe that both architectures perform similarly in standard training regimes where the networks are much wider than they are deep. 

\subsection{Convolutional ResNet CIFAR-10 Experiments}
\label{subsec:cifar}

These experiments investigate how the Balanced ResNet architecture modification, namely randomly flipping between $\vpp$ and $\vp_{-}$, effects training for deep convolutional ResNets (We call these C-ResNets here to distinguish from the fully connected ones studied in detail in the paper). Although the theory in this paper only are proven for Vanilla ResNets (which are fully connected with skip connections), we expect that the Balanced ResNet tweak does not negatively effect performance in standard training regimes and may allow better initialization for very deep models. 

For this experiment,
we used the ResNet implementations from  the GitHub repository by
\citet{kuangliu2017cifar} (MIT License).
To create a Balanced Convolutional ResNet (Balanced C-ResNet), we modified the class
\texttt{PreActBlock} to add flipped ReLUs into the network channel-wise, 
thinking of channels as the natural generalization of neurons for 
convolution neural networks. 
More specifically, we are interested in an input tensor $Y$ of 
dimension $(b, c, h, w)$, where $b$ is for batch size, 
$c$ is for channel, $h$ is for height, and $w$ is for width. 
Before feeding into a ReLU non-linearity, 
we will multiply $Y$ by a vector of 
iid random signs $\{s_j\}_{j \in [c]}$, 
so that the ReLU output is 
\begin{equation}
	[ \varphi_{s_j} ( Y_{i,j,k,l} ) ]_{i,j,k,l} 
	= [ \varphi_{+} ( s_j \, Y_{i,j,k,l} ) ]_{i,j,k,l} \,. 
\end{equation}

\begin{figure}[t]
\centering
\begin{subfigure}[b]{0.495\textwidth}
\centering
\includegraphics[width=\textwidth]{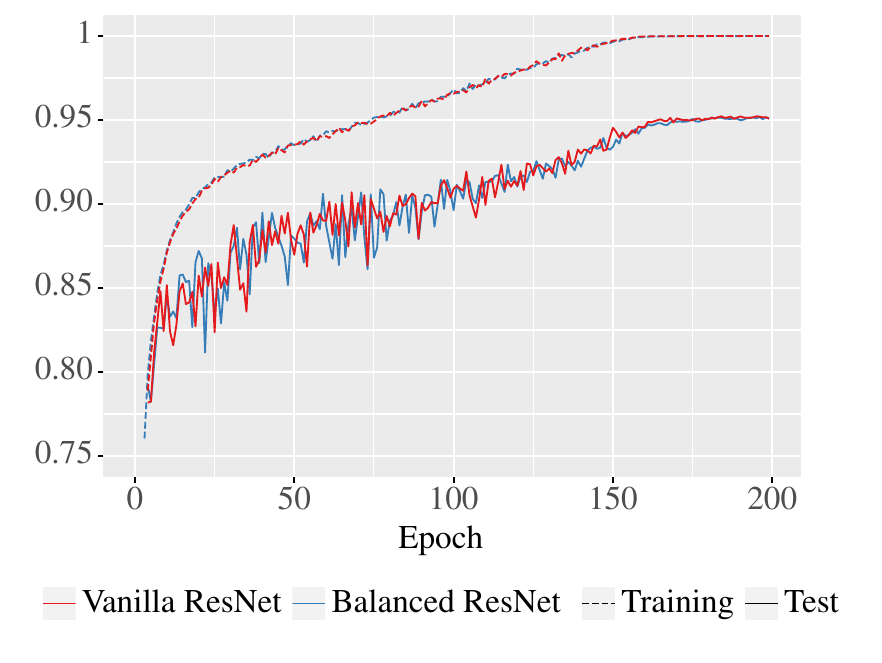}
\caption{
Convolutional ResNet18
}
\label{subfig:C-ResNet18}
\end{subfigure}
\begin{subfigure}[b]{0.495\textwidth}
\centering
\includegraphics[width=\textwidth]{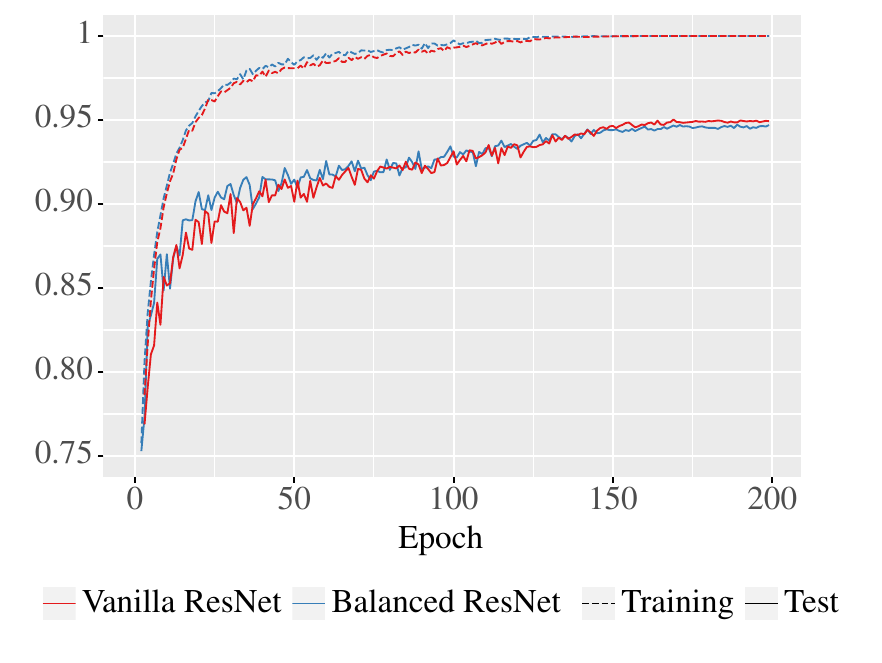}
\caption{
Convolutional ResNet101
}
\label{subfig:C-ResNet101}
\end{subfigure}
\caption{
Pre-activation Convlutional ResNet18 and ResNet101 \citep{he2016identity} 
and corresponding balanced versions. Accuracies on CIFAR-10 are reportedwith learning rate $0.1$, momentum $0.9$, weight decay $5\mathrm{e}{-4}$, 
cosine annealing with $T=200$, 
and batch size $128$. 
}
\label{fig:cifar-resnet101}
\end{figure}

The experiment results are reported in \cref{tb:cifar}
and \cref{fig:cifar-resnet101}. 
Here, we demonstrate that using the Balanced ResNet architecture 
does not reduce the performance of existing training regimes. For the deeper ResNet101, the Balanced ResNet seems to perform slightly better early in training but finishes $0.2\%$ worse at the end of training.

Note that even the deeper ResNet101 tested here is still relatively shallow compared to its ``width''; most of the layers are either $256$ or $1024$ channels by $8 \times 8$ neurons per channel which represent many more neurons in each hidden layer than the depth of the network, $101$. The possible advantage of the Balanced ResNet idea is that it will enable the training  architectures which are even \emph{deeper compared to their width}, which are currently not trainable or difficult to train due to initialization issues. A detailed empirical study exploring this idea is needed to study how Balanced ResNets perform beyond the theory proven in this paper.  

\begin{table}[t]
\centering
\renewcommand{\arraystretch}{1.6}
\begin{tabular}{|l|l|l|}
\hline
Architecture & Test Accuracy \\ \hline
Convolutional ResNet18 & $95.09\%$ \\ \hline
Balanced C-ResNet18 & $95.08\%$ \\ \hline
Convolutional ResNet101 & $94.93\%$ \\ \hline
Balanced C-ResNet101 & $94.70\%$ \\ \hline
\end{tabular}
\caption{
CIFAR-10 \citep{krizhevsky2009learning} 
experiment test accuracies 
with both pre-activation ResNets \citep{he2016identity} 
and the corresponding balanced versions. 
The ResNet18 architectures used a learning rate $0.05$, 
and the ResNet101 used a learning rate of $0.1$. 
Both architectures used momentum $0.9$, 
weight decay $5\mathrm{e}{-4}$, 
cosine annealing with $T=200$, 
and batch size $128$. 
}
\label{tb:cifar}
\end{table}

\subsection{Density Plot Calculations}
\label{subsec:density_calc}

In this section, we describe the calculations required
for plotting \cref{fig:density_merge}.
Firstly, we need to estimate the hypoactivation constant
$C_{\alpha, \lambda}$ from \cref{prop:hypo_and_cov}
using Monte Carlo simulations.
For the choice of $\alpha = \lambda = 1/ \sqrt{2}$,
we find the constant $C_{\alpha, \lambda} \approx -0.876$,
which we use for estimating the mean.
See \cref{fig:hypo_coeff_aff} for further simulations
with varying $\alpha, \lambda$ values, 
and \cref{fig:mean_var_aff} for simulations demonstrating these constants provide accurate prediction for mean and variance. 

\begin{figure}[t]
\centering
\includegraphics[width=0.6\linewidth]{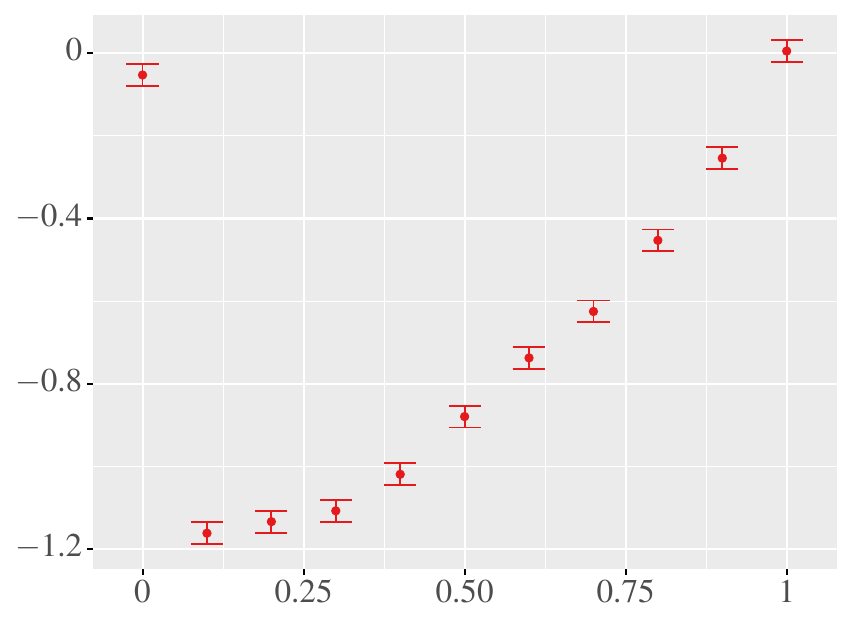}
\caption{
Hypoactivation constant $C_{\al,\la}$
from \cref{prop:hypo_and_cov}
vs. $\lambda^2$.
Here we fix $\alpha^2 + \lambda^2 = 1$,
and $n = d = 150$.
}
\label{fig:hypo_coeff_aff}
\end{figure}

Next, using the choice of $\norm{x_{in}} = 1$
and $\alpha^2 + \lambda^2 = 1$,
we can write
\begin{equation*}
    \ln \norm{\zout}^2
    \dequal
    - \ln n_{in}
    + G
    + \ln \norm{ \vec{Z} }^2 \,.
\end{equation*}

Here we observe that
$\norm{ \vec{Z} }^2 \sim \chi^2(n_{out})$,
and therefore we can compute the density of
$\ln \norm{ \vec{Z} }^2$
with a coordinate change
\begin{equation*}
    \frac{1}{2^{n_{out}/2} \Gamma( n_{out} / 2 ) }
    \exp\left( \frac{n_{out} x}{2} \right)
    \exp\left( \frac{-e^x}{2} \right) \,.
\end{equation*}

Finally, since $G$ and $\ln \norm{\vec{Z}}^2$
are independent,
we can recover the density of $\ln \norm{\zout}^2$
via a numerical convolution
with a Gaussian density of mean $\mathbb{E} G - \ln n_{in}$
and variance $\var(G)$.

To compute the density of the infinite width prediction,
we first observe that in this limit
$\zout \sim \cN(0, \sigma^2 I_{n_{out}})$.
Therefore it's sufficient to simply compute the variance.

To this goal, we follow the calculations of \citet{hayou2021stable} with
the variance recursion formula
(slightly modified to include $\alpha$)
\begin{equation*}
    Q_\ell
    =
        \alpha^2 Q_{\ell-1}
    +
        \lambda^2 \left(
            \sigma_b^2 + \frac{\sigma_w^2}{2}
            \left( 1 + \frac{f(C_\ell)}{ C_\ell } \right)
            Q_{\ell - 1}
        \right)
    = Q_{\ell-1} \,,
\end{equation*}
where we plugged in values of $\sigma_b = 0$,
$\sigma_w = 2$,
and $C_\ell = 1$ and $f(C_\ell) = 0$.
To complete the recursion,
the initial $Q_0 = \frac{1}{n_{in}} \norm{x}^2$.
This implies that
$\zout \sim \cN(0, \frac{1}{n_{in}} I_{n_{out}}) $,
and therefore
$\norm{\zout}^2 \sim \frac{1}{n_{in}} \chi^2(n_{out})$,
which implies $\ln \norm{\zout}^2$ has density
\begin{equation*}
    \frac{1}{2^{n_{out}/2} \Gamma( n_{out} / 2 ) }
    n_{in}^{n_{out} / 2}
    \exp\left( \frac{n_{out} x}{2} \right)
    \exp\left( \frac{- n_{in} e^x }{2} \right) \,.
\end{equation*}

\subsection{Additional Monte Carlo Simulations}
\label{subsec:extra_simulations}

These additional Monte Carlo simulations provide further comparisons between the infinite depth-and-width limit predictions and finite networks. See %
\cref{fig:mean_var_aff}.

\begin{figure}[h]
\centering
\includegraphics[width = 0.8\linewidth]{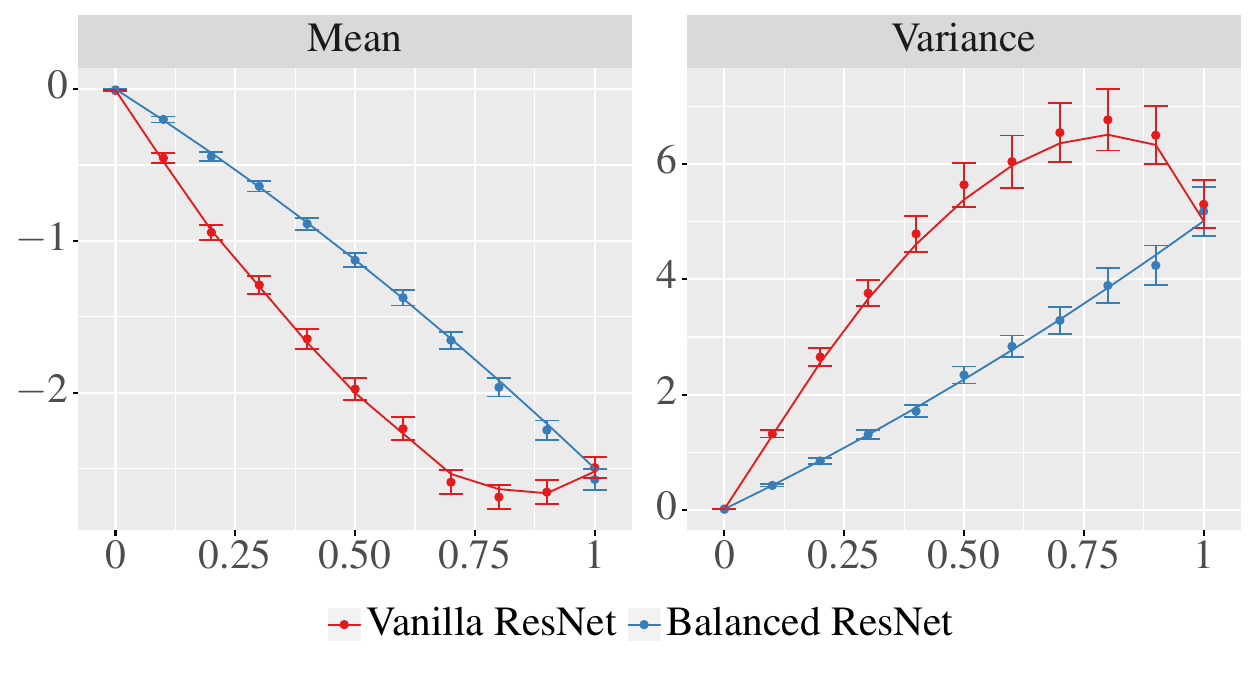}
\caption{Mean and variance of $G(n,d,\al,\lambda)$ 
from \cref{prop:main} 
vs. $\lambda^2$, where $\alpha^2 + \lambda^2 = 1$ and
$n = d = 150$.
The solid lines indicate our prediction using
the infinite-depth-and-width limit. This shows the excess variance appearing in Vanilla vs Balanced ResNets. Also note that fully connected networks with no skip connections, corresponding to $\la=1$, have the highest variance for all balanced ResNets. 
}
\label{fig:mean_var_aff}
\end{figure}

\end{document}